\documentclass{custom} 

\title[Fisher information lower bounds for sampling]{Fisher information lower bounds\\ for sampling}
\usepackage{bbm}
\usepackage{enumitem}
\usepackage{graphicx}
\usepackage{mathtools}
\usepackage{microtype}
\usepackage{tikz}
\usepackage{times}
\usepackage[Symbolsmallscale]{upgreek}
\usepackage{thm-restate} 
\usepackage{style}
\usepackage{mdframed}
\usepackage{amsmath}
\usepackage{euscript}

\makeatletter
\def\set@curr@file#1{\def\@curr@file{#1}} 
\makeatother
\usepackage[load-configurations=version-1]{siunitx} 

\usepackage{float}

\usepackage{mathrsfs} 
\usepackage{style}

\newcommand{\ve}[1]{\left\Vert {#1}\right\Vert}

\newcommand{\be}[0]{\beta}

\newcommand{\ep}[0]{\epsilon}

\newcommand{\rc}[1]{\frac{1}{#1}}

\newcommand{\fc}[2]{\frac{#1}{#2}}

\newcommand{\nb}[0]{\nabla}

\coltauthor{\Name{Sinho Chewi}\thanks{Part of this work was done while SC was a research intern at Microsoft Research.} \Email{schewi@mit.edu}\\ \addr MIT \AND
 \Name{Patrik Gerber} \Email{prgerber@mit.edu}\\ \addr MIT \AND
 \Name{Holden Lee} \Email{hlee283@jhu.edu} \\
 \addr Johns Hopkins University \AND
  \Name{Chen Lu} \Email{chenl819@mit.edu} \\
 \addr MIT 
}

\begin{document}

\maketitle

\begin{abstract}

We prove two lower bounds for the complexity of non-log-concave sampling within the framework of Balasubramanian et al.\ (2022), who introduced the use of Fisher information ($\FI$) bounds as a notion of approximate first-order stationarity in sampling. Our first lower bound shows that averaged LMC is optimal for the regime of large $\FI$ by reducing the problem of finding stationary points in non-convex optimization to sampling. Our second lower bound shows that in the regime of small $\FI$, obtaining a $\FI$ of at most $\varepsilon^2$ from the target distribution requires $\poly(1/\varepsilon)$ queries, which is surprising as it rules out the existence of high-accuracy algorithms (e.g., algorithms using Metropolis{--}Hastings filters) in this context.

\end{abstract}

\begin{keywords}%
  Fisher information, gradient descent, Langevin Monte Carlo, non-log-concave sampling, sampling lower bound, stationary point
\end{keywords}



\section{Introduction}

What is the query complexity of sampling from a $\beta$-log-smooth but possibly non-log-concave target distribution $\pi$ on $\R^d$? Until recently, this question was only investigated from an upper bound perspective, and only for restricted classes of distributions, such as distributions satisfying functional inequalities~\citep{vempala2019ulaisoperimetry, wibisono2019proximal, maetal2021nesterovmcmc, chewietal2022lmcpoincare}, distributions with tail decay conditions~\citep{durmus2017nonasymptotic, chengetal2018noncvx, xuetal2018globalconv, lietal2019rungekutta, majmijszp2020withoutlogconcavity, erdhos2021interplay, zouxugu2021sglangevin, hebalerd2022heavytailed}, or mixtures of log-concave distributions~\citep{lee2018beyond}.

In the recent work ~\citet{balasubramanianetal2022nonlogconcave}, a general framework for the investigation of non-log-concave sampling was developed. Motivated by stationary point analysis in non-convex optimization~\citep[see, e.g.,][]{nesterov2018lectures} and the interpretation of sampling as optimization over the space of probability measures~\citep{jko1998, wibisono2018samplingoptimization}, Balasubramanian et al.\ proposed to call any measure $\mu$ satisfying $\sqrt{\FI(\mu \mmid \pi)} \le \varepsilon$ an $\varepsilon$-stationary point for sampling, where $\FI(\mu \mmid \pi) \deq \E_{\mu}[\norm{\nb \ln \fc{\mu}{\pi}}^2]$ denotes the \emph{relative Fisher information} of $\mu$ from $\pi$. They explained the interpretation of this condition via the classical phenomenon of \emph{metastability}~\citep{bovetal2002metastability, bovetal2004metastabilityi, bovgaykle2005metastabilityii}, and showed that averaged Langevin Monte Carlo (LMC) can find an $\varepsilon$-stationary point in $\mc O(\beta^2 dK_0/\varepsilon^4)$ iterations, where $K_0 \deq \KL(\mu_0 \mmid \pi)$ is the initial Kullback{--}Leibler (KL) divergence to the target $\pi$.

In the field of optimization, however, there are also corresponding \emph{lower bounds} on the complexity of finding stationary points~\citep{vavavis1993complexity, nesterov2012make, bubmik2020gradientflow, carmonetal2020stationarypt, caretal2021statptii, chewibubecksalim2022stationary}. Such lower bounds are important for identifying optimal algorithms and understanding the fundamental difficulty of the task at hand. For example, the work of~\citet{carmonetal2020stationarypt} shows that the standard gradient descent algorithm is optimal for finding stationary points of smooth functions, at least in high dimension.

In this work, we establish the first lower bounds for Fisher information guarantees for sampling, resolving an open question posed in~\citet{balasubramanianetal2022nonlogconcave}. As we discuss further below, our results also reveal a surprising equivalence between the task of obtaining a sample which has moderate Fisher information relative to a target distribution and the task of finding an approximate stationary point of a smooth function, thereby strengthening the connection between the fields of non-convex optimization and non-log-concave sampling.

\paragraph{Our contributions.} We now informally describe our main results. Let $\pi \propto \exp(-V)$ denote the target density on $\R^d$, where $V : \R^d\to\R$ is called the \emph{potential}. Throughout, our notion of complexity is the minimum number of queries made to an oracle that returns the value of $V$ (up to an additive constant) and its gradients.

Our first result connects the task of obtaining Fisher information guarantees with finding stationary points in non-convex optimization, for a particular regime of large inverse temperature $\beta$.

\begin{theorem}[equivalence, informal]\label{thm:equivalence_informal}
    Let $V : \R^d\to\R$ be a smooth function and for $\beta > 0$, let $\pi_\beta \propto \exp(-\beta V)$, where $\pi_\beta$ is well-defined (i.e., $\int \exp(-\beta V) < \infty$). Then, the following two problems are equivalent.
    \begin{enumerate}
        \item Output an $\varepsilon$-stationary point of $V$.
        \item Output a sample from a measure $\mu$ such that $\FI(\mu \mmid \pi_\beta) \lesssim \beta d$, where $\beta \asymp d/\varepsilon^2$.
    \end{enumerate}
\end{theorem}

By combining this equivalence with the lower bound of~\citet{carmonetal2020stationarypt}, we obtain:

\begin{theorem}[first lower bound, informal]
    The number of queries required to obtain a sample from a measure $\mu$ satisfying $\sqrt{\FI(\mu \mmid \pi)} \lesssim \sqrt{\beta d}$, starting from an initial distribution $\mu_0$ with KL divergence $K_0 \deq \KL(\mu_0 \mmid \pi)$, is at least $\Omega(K_0/d)$. The lower bound is attained by averaged LMC~\citep{balasubramanianetal2022nonlogconcave}.
\end{theorem}

The first lower bound addresses the regime of large Fisher information, $\FI(\mu \mmid \pi) \lesssim \beta d$. In order to address the regime of small Fisher information, $\FI(\mu\mmid \pi) \lesssim \varepsilon^2$, we give a construction based on hiding a bump of large mass and prove the following:

\begin{theorem}[second lower bound, informal]\label{thm:second_lower_informal}
    The number of queries required to obtain a sample from a measure $\mu$ satisfying $\sqrt{\FI(\mu \mmid \pi)} \le \varepsilon$, starting from an initial distribution $\mu_0$ with KL divergence $K_0 \deq \KL(\mu_0 \mmid \pi) \le 1$, is at least ${(\sqrt\beta/\varepsilon)}^{1 \vee (2d/(d+2)) - o(1)}$.
\end{theorem}

We give a more precise form of our lower bound in Section~\ref{scn:second_lower_bd_main_text}.
In infinite dimension (actually, $d\ge \widetilde \Omega(\sqrt{\log(\beta/\varepsilon^2)})$ suffices, see Section~\ref{scn:second_lower_bd_main_text}), the lower bound reads $\widetilde \Omega(\beta/\varepsilon^2)$, which can be compared to the averaged LMC upper bound of $\mc O(\beta^2 d/\varepsilon^4)$. It is an open question to close this gap.

We next discuss implications of our results.
\begin{itemize}
    \item \textbf{Towards a theory of lower bounds for sampling.} The problem of obtaining sampling lower bounds is a notorious open problem raised in many prior works~\citep[see, e.g.,][]{chengetal2018underdamped, ge2020estimating, leeshetia2021malalower, chatterjibartlettlong2022oraclesampling}. So far, unconditional lower bounds have only been obtained in restricted settings such as in dimension $1$; see~\citet{chewietal2022logconcave1d} and the discussion therein, as well as the reduction to optimization in~\citet{gopleeliu2022privatecvxopt}. Our lower bounds are the first of their kind for Fisher information guarantees, and are some of the \emph{only} lower bounds for sampling in general. Hence, our work takes a significant step towards a better understanding of the complexity of sampling. In particular, our first lower bound identifies a regime in which (averaged) LMC is \emph{optimal}, which was not previously known in any setting.
    \item \textbf{Stronger connections between non-convex optimization and non-log-concave sampling.} The equivalence in Theorem~\ref{thm:equivalence_informal} provides compelling evidence that Fisher information guarantees are the correct analogue of stationary point guarantees in non-convex optimization, thereby 
    supporting the framework of~\citet{balasubramanianetal2022nonlogconcave}.
    \item \textbf{Obtaining an approximate stationary point in sampling is strictly harder for non-log-concave targets.} Ignoring the dependence on other parameters besides the accuracy, our second lower bound yields a $\poly(1/\varepsilon)$ lower bound for the Fisher information task for non-log-concave targets.
    In contrast, it is morally possible to solve this task in $\polylog(1/\varepsilon)$ queries for \emph{log-concave} targets; see Appendix~\ref{scn:separation} for justification. This exhibits a stark separation between log-concave and non-log-concave sampling. Note that the analogous separation does not exist in the context of optimization, because there is a $\poly(1/\varepsilon)$ lower bound for finding an $\varepsilon$-stationary point of a convex and smooth function~\citep{caretal2021statptii}.
    \item \textbf{A separation between optimization and sampling.} Finally, our second lower bound yields a $\poly(1/\varepsilon)$ lower bound, even in dimension one. In contrast, for the analogous question in optimization of finding an $\varepsilon$-stationary point of a univariate function, the recent work of~\citet{chewibubecksalim2022stationary} exhibits an algorithm with $\mc O(\log(1/\varepsilon))$ complexity. To our knowledge, this is one of the first instances in which sampling is provably harder than optimization.
\end{itemize}

\section{Notation and setting}\label{scn:setting}
\paragraph{Notation.}
Given a probability measure $\pi$ on $\R^d$ which admits a density w.r.t.\ the Lebesgue measure, we abuse notation by identifying $\pi$ with its density.

The class of distributions that we wish to sample from are the $\beta$-log-smooth distributions on $\R^d$, defined as follows:
\begin{definition}[log-smooth distributions]
The class of $\beta$-log-smooth distributions consists of distributions $\pi$ supported on $\R^d$ whose densities are of the form $\pi \propto \exp(-V)$, for potential functions $V:\R^d \to \R$ that are twice continuously differentiable, and satisfy
\begin{align*}
    \norm{\nabla V(x) - \nabla V(y)} \le\beta\, \norm{x - y}\, , \ \ \ \forall x, y \in \R^d\, .
\end{align*}
\end{definition}

\paragraph{Oracle model.} We work under the following oracle model. The algorithm is given access to a target distribution $\pi$ in our class via two oracles: initialization and local information.
The initialization oracle outputs samples from some distribution $\mu_0$ for which $\KL(\mu_0 \mmid \pi) \le K_0$.
The local oracle for $\pi$, given a query point $x \in \R^d$, returns the value of the potential (up to an additive constant) and its gradient at the query point $x$, i.e., the tuple $(V(x), \nabla V(x))$. Algorithms can access samples from $\mu_0$ for free, and we care about the number of local information queries needed. The query complexity is defined as follows.

\begin{definition}[query complexity]
Let $\ms C(d, K_0, \epsilon; \beta)$ be the smallest number $n\in\N$ such that any algorithm which works in the oracle model described above and outputs a sample from a measure $\mu$ satisfying $\sqrt{\FI(\mu \mmid \pi)} \le \varepsilon$, for any $\beta$-log-smooth target $\pi$ and any valid initialization oracle for $\pi$, requires at least $n$ queries to the local oracle for $\pi$.
\end{definition}

The upper bound of~\citet{balasubramanianetal2022nonlogconcave} shows that using averaged LMC,
\begin{align}\label{eq:compl_upper_bd}
    \ms C(d,K_0,\varepsilon; \beta)
    &\lesssim 1 \vee \frac{\beta^2 dK_0}{\varepsilon^4}\,.
\end{align}

We also note the following rescaling lemma.

\begin{lemma}[rescaling]\label{lem:rescaling}
    It holds that
    \begin{align*}
        \ms C(d, K_0, \varepsilon;\beta)
        &= \ms C\bigl(d, K_0, \frac{\varepsilon}{\sqrt \beta}; 1\bigr)\,.
    \end{align*}
\end{lemma}
\begin{proof}
    Suppose that $V : \R^d\to\R$ is $\beta$-smooth and $\pi\propto\exp(-V)$.
    Define the rescaled potential $V_\beta : \R^d\to\R$ via $V_\beta(x) \deq V(x/\sqrt\beta)$, and let $\pi_\beta \propto \exp(-V_\beta)$; in particular, if $Z \sim \pi$, then $\sqrt\beta \, Z \sim \pi_\beta$. Then, $V_\beta$ is $1$-smooth.
    Moreover, suppose $\KL(\mu \mmid \pi) = K_0$ and that $X \sim \mu$ is a sample from $\mu$; let $\mu_\beta \deq \law(\sqrt \beta \, X)$.
    Since the KL divergence is invariant under bijective transformations, we have $\KL(\mu_\beta \mmid \pi_\beta) = K_0$, which shows that we can simulate an initialization oracle for $\pi_\beta$ given an initialization oracle for $\pi$.
    We can also simulate the local oracle for $\pi_\beta$ given a local oracle for $\pi$, as $\nabla V_\beta(x) = \rc{\sqrt{\be}}\nabla V(x/\sqrt\be)$. Finally, if $\hat\mu_\beta$ satisfies $\sqrt{\FI(\hat\mu_\beta \mmid \pi_\beta)} \le \varepsilon/\sqrt \beta$ and $\hat X_\beta \sim \hat\mu_\beta$, let $\hat \mu \deq \law(\hat X_\beta/\sqrt\beta)$. A straightforward calculation shows that $\sqrt{\FI(\hat\mu\mmid \pi)} \le \varepsilon$. This proves the upper bound $\ms C(d, K_0,\varepsilon;\beta) \le \ms C(d, K_0, \varepsilon/\sqrt\beta; 1)$, and the reverse bound follows because this reduction is reversible.
\end{proof}

From here on, we abbreviate $\ms C(d,K_0,\varepsilon) \deq \ms C(d, K_0, \varepsilon; 1)$.

\section{Reduction to optimization and the first lower bound}

In this section, we show a perhaps surprising equivalence between obtaining Fisher information guarantees in sampling and finding stationary points of smooth functions in optimization. The formal statement of the equivalence is as follows.

\begin{theorem}[equivalence]\label{thm:equivalence}
    Let $V : \R^d\to\R$ be a $1$-smooth function such that for any $\beta > 0$, the function $\exp(-\beta V)$ is integrable. 
    Let $\pi_\beta$ be the probability measure with density $\pi_\beta \propto \exp(-\beta V)$, where $\beta = d/\varepsilon^2$.
    \begin{enumerate}
        \item Suppose that $x \in \R^d$ is a point with $\norm{\nabla V(x)} \le \varepsilon$.
        Then, for $\mu_\beta \deq \normal(x, \beta^{-1} I_d)$, it holds that $\FI(\mu_\beta \mmid \pi_\beta) \le 10\beta d$.
        \item Conversely, suppose that $\mu$ is a probability measure on $\R^d$ such that $\FI(\mu \mmid \pi_\beta) \le \beta d$. Let $X \sim \mu$ be a sample.
        Then, $\norm{\nabla V(X)} \le 3\varepsilon$ with probability at least $1/2$.
    \end{enumerate}
\end{theorem}
\begin{proof}
    See Appendix~\ref{scn:equivalence_pf}.
\end{proof}

Note that an oracle for $\beta V$ can be simulated from an oracle for $V$, so that the above theorem provides an exact equivalence between a sampling problem and an optimization problem within the oracle model, up to universal constants.

As a first application of this equivalence, we observe that averaged LMC yields an nearly optimal algorithm for finding stationary points of smooth functions. We recall the LMC algorithm for sampling from a density $\pi \propto \exp(-V)$. We fix a step size $h > 0$, initialize at $X_0 \sim \mu_0$, and for $t \in [kh,(k+1)h]$, we set
\begin{align}\label{eq:lmc}
    X_t
    &= X_{kh} - (t-kh) \, \nabla V(X_{kh}) + \sqrt 2 \, (B_t - B_{kh})\,,
\end{align}
where ${(B_t)}_{t\ge 0}$ is a standard Brownian motion in $\R^d$.
Let $\mu_t \deq \law(X_t)$ denote the law of the algorithm at time $t$. Then, the \emph{averaged} LMC algorithm at iteration $N$ outputs a sample from the law of $\bar \mu_{Nh} \deq {(Nh)}^{-1}\int_0^{Nh} \mu_t \, \D t$. This is obtained algorithmically as follows: first, we sample a time $t \in [0,Nh]$ uniformly at random (independently of all other random variables).
Let $k$ denote the largest integer such that $kh \le t$. We then compute $X_0, X_h, X_{2h},\dotsc,X_{kh}$ using the LMC recursion, and then output $X_t$ which is obtained via the partial LMC update~\eqref{eq:lmc}.

\begin{corollary}[averaged LMC is nearly optimal for finding stationary points]\label{cor:averaged_lmc}
    Let $V : \R^d\to\R$ be $1$-smooth and satisfy $V(0) - \inf V \le \Delta$.
    Let $\varepsilon > 0$ be such that $\Delta/\varepsilon^2 \ge 1$.
    Assume that for $\beta = d/\varepsilon^2$, the probability measure with density $\pi_\beta \propto \exp(-\beta V)$ is well-defined and that $\int \norm\cdot^2 \, \D \pi_\beta \le \poly(\Delta, d, 1/\varepsilon)$.
    Consider running averaged LMC with step size $h = \widetilde \Theta(1/\beta)$, initial distribution $\mu_0 = \normal(0, \beta^{-1} I_d)$, and target $\pi_\beta$, with
    \begin{align*}
        N
        &\ge \widetilde \Omega\Bigl( \frac{\Delta}{\varepsilon^2} \Bigr) \qquad\text{iterations}\,.
    \end{align*}
    Then, we obtain a sample $X$ such that with probability at least $1/2$, it holds that $\norm{\nabla V(X)} \lesssim \varepsilon$.
\end{corollary}
\begin{proof}
    We combine Theorem~\ref{thm:equivalence} with the analysis of averaged LMC in~\citet{balasubramanianetal2022nonlogconcave}; see Appendix~\ref{scn:averaged_lmc_pf}.
\end{proof}

This matches the usual $\mc O(\Delta/\varepsilon^2)$ complexity for the standard gradient descent algorithm to find an $\varepsilon$-stationary point~\citep[see, e.g.,][]{bubeck2015convex, nesterov2018lectures}. On its own, this observation is not terribly surprising because as $\beta\to\infty$, the LMC iteration~\eqref{eq:lmc} recovers the gradient descent algorithm. However, it is remarkable that the analysis of~\citet{balasubramanianetal2022nonlogconcave} of averaged LMC in Fisher information nearly recovers the gradient descent guarantee.

This observation also suggests that the lower bound of~\citet{carmonetal2020stationarypt}, which establishes optimality of gradient descent for finding stationary points in high dimension, also implies optimality of averaged LMC in a certain regime. We obtain the following theorem.

\begin{theorem}[first lower bound]\label{thm:first_lower}
    Suppose that the dimension $d$ satisfies $\widetilde{\mc O}(K_0) \ge d \ge \widetilde \Omega(K_0^{2/3})$. Then, it holds that
    \begin{align*}
        \ms C\bigl(d, K_0, \varepsilon = \sqrt{\beta d}; \beta\bigr)
        &\gtrsim \frac{K_0}{d}\,.
    \end{align*}
\end{theorem}
\begin{proof}
    See Appendix~\ref{scn:first_lower_pf}.
\end{proof}

The lower bound of Theorem~\ref{thm:first_lower} is matched by averaged LMC, see~\eqref{eq:compl_upper_bd}. In the theorem, the restriction $d\ge \widetilde \Omega(K_0^{2/3})$ arises because the lower bound construction of~\citet{carmonetal2020stationarypt} for finding a $\epsilon$-stationary point 
of a smooth function requires a large dimension $d\ge \widetilde \Omega(1/\epsilon^4)$. If, as conjectured in~\citet{bubmik2020gradientflow, chewibubecksalim2022stationary}, the lower bound construction can be embedded in dimension $d \gtrsim \log(1/\epsilon)$, then the restriction in Theorem~\ref{thm:first_lower} would instead become $d\gtrsim \log K_0$.

\section{Bump construction and the second lower bound}\label{scn:second_lower_bd_main_text}

The main drawback of the first lower bound (Theorem~\ref{thm:first_lower}) is that it only provides a lower bound on the Fisher information for a specific value of the target accuracy, $\varepsilon = \sqrt{\beta d}$.
To complement this result, we provide the following lower bound for the query complexity of sampling to high accuracy in Fisher information; recall that it suffices to consider $\beta = 1$ by the rescaling lemma (Lemma~\ref{lem:rescaling}).

\begin{theorem}[second lower bound]\label{thm:bump_high_dim}
    For the class of $1$-log-smooth distributions on $\R^d$, there exist universal constants $c, c' > 0$, such that for all $\epsilon < \exp(-c' d)$, we have
\begin{align}\label{eq:second_lower}
    \ms C(d, K_0=1, \epsilon) \gtrsim
    \Bigl( \frac{cd}{\log(1/\varepsilon)}\Bigr)^{d/2} \, \frac{1}{\varepsilon^{2d/(d+2)}}\,.
\end{align}
\end{theorem}
\begin{proof}
    See Appendix~\ref{sec:pf_lw_bd_2nd}.
\end{proof}

Here, we sketch the main ideas of the proof. We construct a family of distributions in our class which contain a constant fraction of their mass on disjoint bumps.
We reduce the task of estimation to the task of sampling by showing that if an algorithm can sample well in Fisher information uniformly over our class of $1$-log-smooth distributions, then it can be used to identify the location of the bump.
We then use an information theoretic argument to lower bound the number of queries that any randomized algorithm needs to solve the latter task. Although the scheme of the argument is straightforward, the actual proof requires some delicate calculations, involving asymptotics of integrals and a careful balancing of parameters.

The lower bound in Theorem~\ref{thm:bump_high_dim} deteriorates in high dimension; note that due to the restriction $\varepsilon \le \exp(-c'd)$, the first factor in~\eqref{eq:second_lower} is exponentially small in $d$. However, we can remedy this by noting that a $d$-dimensional construction can be embedded into $\R^{d'}$ for any $d' \ge d$, and hence
\begin{align*}
    \ms C(d, K_0=1, \varepsilon)
    &\gtrsim \max_{d_\star \le d}{\Bigl[ \Bigl( \frac{cd_\star}{\log(1/\varepsilon)}\Bigr)^{d_\star/2} \, \varepsilon^{4/(d_\star + 2)} \Bigr]} \, \frac{1}{\varepsilon^2}\,.
\end{align*}
By optimizing over $d_\star$, we show (Appendix~\ref{scn:optimize_bd}) that if $\varepsilon \le 1/C$, then
\begin{align*}
    \ms C(d, 1, \varepsilon)
    &\gtrsim \frac{1}{\varepsilon^2 \exp(C\sqrt{\log(1/\varepsilon) \log\log(1/\varepsilon)})}
    = \frac{1}{\varepsilon^{2 - o(1)}}\,, \qquad\text{for all}~d\gtrsim \sqrt{\frac{\log(1/\varepsilon)}{\log\log(1/\varepsilon)}}\,,
\end{align*}
where $C > 0$ is universal.

For $d=1$, the lower bound of Theorem~\ref{thm:bump_high_dim} reads $\ms C(1,1,\varepsilon)\gtrsim 1/(\varepsilon^{2/3} \sqrt{\log(1/\varepsilon)})$. However, for the one-dimensional case we can in fact obtain better bounds on the Poincar\'e constants of the measures in our lower bound construction, leading to an improvement of the exponent from $2/3$ to $1$.
This result is stated below.

\begin{theorem}[second lower bound, univariate case]\label{thm:second_lower_1d}
    For the class of $1$-log-smooth distributions on $\R$, there exists a universal constant $c > 0$, such that for all $\epsilon < c$, we have
    \begin{align*}
        \ms C(d=1, K_0=1, \epsilon) \gtrsim \frac{1}{\epsilon\sqrt{\log(1/\varepsilon)}}\, .
    \end{align*}
\end{theorem}
\begin{proof}
    See Appendix~\ref{scn:pf_second_lower_1d}.
\end{proof}

The univariate setting also provides a convenient setting in order to compare our lower bounds with algorithms such as rejection sampling, so we include a detailed discussion in Appendix~\ref{scn:univariate_discussion}. We highlight a few interesting conclusions of the discussion here.
\begin{itemize}
    \item Although rejection sampling can indeed obtain Fisher information guarantees with complexity $\mc O(\log(1/\varepsilon))$ (Proposition~\ref{prop:rej_sampling_fi}), this does not contradict our lower bounds because rejection sampling cannot be directly implemented within our oracle model. Instead of an initialization $\mu_0$ satisfying $\KL(\mu_0 \mmid \pi) \le K_0$, rejection sampling requires the stronger assumption $\max\{\sup \ln(\mu_0/\pi), \sup \ln(\pi/\mu_0)\} \le M_0$. Under this stronger initialization oracle, the complexity guarantee for rejection sampling is $\mc O(\exp(3M_0) \log(1/\varepsilon))$.
    \item In the model with the stronger initialization oracle (i.e., bounded $M_0$), any algorithm which has $\polylog(1/\varepsilon)$ dependence on the accuracy $\varepsilon$ necessarily incurs exponential dependence on $M_0$ (Corollary~\ref{cor:exponential_m0}). This demonstrates a fundamental trade-off between high accuracy (e.g., rejection sampling) and polynomial dependence on $M_0$ (e.g., averaged LMC). 
    \item The initialization oracle with bounded $M_0$ is strictly stronger than the one with bounded $K_0$. In other words, sampling is strictly easier in the presence of an initialization with bounded density ratio to the target (i.e., a \emph{warm start}) than an initialization with bounded KL divergence. This is consistent with intuition from prior work on the complexity of the Metropolis-adjusted Langevin algorithm~\citep[see][]{chewietal2021mala, leeshetia2021malalower, wuschmidlerchen2021mala}.
    \item The effective radius $R$ of our lower bound construction scales with $1/\varepsilon$. This is in fact necessary: if $R$ is fixed then there is an algorithm with $\mc O(\log(1/\varepsilon))$ complexity (Proposition~\ref{prop:bdd_supp}).
\end{itemize}

\section{Conclusion}

In this work, we have provided the first lower bounds for the query complexity of obtaining Fisher information guarantees for sampling. Due to the scarcity of general sampling lower bounds, our bounds are in fact some of the \emph{only} known lower bounds for sampling. Our results have a number of interesting implications, which we discussed thoroughly in previous sections, and they advance our understanding of the fundamental task of non-log-concave sampling.

To conclude, we highlight a few problems left open in our work. Most notably, our lower bound in Theorem~\ref{thm:bump_high_dim} does not match the upper bound of averaged LMC, and it is an important question to close this gap. We also note that our lower bounds in Theorems~\ref{thm:bump_high_dim} and~\ref{thm:second_lower_1d} do not capture the dependence of $K_0$, and this is also left for future work.

\acks{We thank S\'ebastien Bubeck, Adil Salim, and Ruoqi Shen for helpful discussions. SC was supported by the Department of Defense (DoD) through the National Defense Science \& Engineering Graduate Fellowship (NDSEG) Program, as well as the NSF TRIPODS program (award DMS-2022448). CL was supported by the Eric and Wendy Schmidt Center. PG was supported by NSF award IIS-1838071. }


\newpage
\appendix

\section{Separation between log-concave and non-log-concave sampling}\label{scn:separation}

We show that $\mc O(\log\frac{1}{\varepsilon})$ Fisher information query complexity is attainable for log-concave densities, by giving a generic post-processing method to turn $\chi^2$-error guarantees into Fisher information guarantees.

\subsection{Post-processing lemma}

Let $Q_t$ denote heat flow for time $t$  (i.e., convolution with a Gaussian of variance $t$). We aim to bound $\FI(\mu Q_t \mmid \pi)$, where $\pi$ is the distribution that we wish to sample from, and $\mu$ is the output of a sampling algorithm with chi-squared error guarantees.

\begin{lemma}[Fisher information guarantee from a chi-squared guarantee]
\label{l:FI-heat-bd}
    Suppose that $\mu$ and $\pi$ are two probability measures on $\R^d$, that $\pi$ is $\beta$-log-smooth, and that $\chi^2(\mu \mmid \pi) \le \varepsilon_\chi^2 \le 1$.
    Then, if $t \lesssim 1/\beta$ for a small enough implied constant, it holds that
    \begin{align*}
        \FI(\mu Q_t \mmid \pi)
        &\lesssim \frac{\varepsilon_\chi \, (d+\log(1/\varepsilon_\chi))}{t} + \beta^2 dt\,.
    \end{align*}
\end{lemma}

To prove Lemma~\ref{l:FI-heat-bd}, we start with
\begin{align}
\label{e:FI-decomp}
    \FI(\mu Q_t \mmid \pi)
    &\deq \int \norm{\nabla \ln(\mu Q_t)(x) - \nabla \ln \pi(x)}^2 \, \mu Q_t(\D x) \nonumber \\
    &\le 2 \FI (\mu Q_t \mmid \pi Q_t) + 2\int_{\R^d} \norm{\nb \log(\pi Q_t)(x) - \nb \log \pi(x)}^2 \,\mu Q_t(\D x)\,.
\end{align}
For the first term in~\eqref{e:FI-decomp}, we use the following lemma on error in the score function (gradient of the log-density).

\begin{lemma}[{score error under heat flow,~\citet[Lemma 6.2]{lee2022b}}]\label{lem:l2_score_error-chi}
Let $\mu$ and $\pi$ be probability measures on $\R^d$, and let $Q_t$ denote the heat semigroup at time $t$. 
In addition, we assume that $\chi^2(\mu \mmid \pi)\le \ep_\chi^2\le 1$. 
Then,
\begin{align*}
\FI(\mu Q_t\mmid \pi Q_t) =
    \int_{\R^d}\norm{\nb \ln (\mu Q_t)(x) - \nb \ln (\pi Q_t)(x)}^2 \,\mu Q_t(\D x) 
    \lesssim  \fc{\ep_\chi \, \bigl(d+\ln\frac{1}{\varepsilon_\chi}\bigr)}{t}\,.
\end{align*}
\end{lemma}
For the second term in~\eqref{e:FI-decomp}, we use the following score perturbation lemma.

\begin{lemma}[{\citet[Lemma C.11]{lee2022convergence}}]\label{l:perturb-gaussian}
Suppose that $\pi\propto\exp(-V)$ is a probability density on $\R^d$, where $V$ is $\beta$-smooth. 
Then for $\beta\le \rc{2t}$,
\[
\Bigl\lVert \nb \ln \fc{\pi(x)}{(\pi Q_t)(x)}\Bigr\rVert \le 6\beta d^{1/2} t^{1/2} + 2\beta t \ve{\nb V(x)}.
\]
\end{lemma}

We are now ready to prove Lemma~\ref{l:FI-heat-bd}.

\medskip{}

\begin{proof}[Proof of Lemma~\ref{l:FI-heat-bd}]
    For the second term in~\eqref{e:FI-decomp}, Lemma~\ref{l:perturb-gaussian} yields
    \begin{align*}
        \E_{\mu Q_t}[\norm{\nabla \ln(\pi Q_t) - \nabla \ln \pi}^2]
        &\lesssim \beta^2 dt + \beta^2 t^2 \E_{\mu Q_t}[\norm{\nabla V}^2]\,.
    \end{align*}
    On the other hand, Lemma~\ref{lem:fisher_info_lem} below yields
    \begin{align*}
        \E_{\mu Q_t}[\norm{\nabla V}^2]
        &\lesssim \FI(\mu Q_t \mmid \pi) + \beta d\,.
    \end{align*}
    Hence, from~\eqref{e:FI-decomp} and Lemma~\ref{lem:l2_score_error-chi},
    \begin{align*}
        \FI(\mu Q_t \mmid \pi)
        &\lesssim \FI(\mu Q_t \mmid \pi Q_t) + \E_{\mu Q_t}[\norm{\nabla \ln(\pi Q_t) - \nabla \ln \pi}^2] \\
        &\lesssim \frac{\varepsilon_\chi \, (d+\log(1/\varepsilon_\chi))}{t} + \beta^2 dt + \beta^2 t^2 \FI(\mu Q_t \mmid \pi)\,.
    \end{align*}
    If $t\lesssim 1/\beta$ for a small enough implied constant, it implies
    \begin{align*}
        \FI(\mu Q_t \mmid \pi)
        &\lesssim \frac{\varepsilon_\chi \, (d+\log(1/\varepsilon_\chi))}{t} + \beta^2 dt
    \end{align*}
    as desired.
\end{proof}

\subsection{High-accuracy Fisher information guarantees for log-concave targets}

We now apply the post-processing lemma (Lemma~\ref{l:FI-heat-bd}). We recall the following high-accuracy guarantee for sampling from log-concave targets in chi-squared divergence, based on the proximal sampler.

\begin{theorem}[{\citet[Corollary 7]{chenetal2022proximalsampler}}]\label{thm:prox_sampler}
    Suppose that the target distribution $\pi \propto\exp(-V)$ is $\beta$-log-smooth and satisfies a Poincar\'e inequality with constant $\CPI$. Then, the proximal sampler, with rejection sampling implementation of the restricted Gaussian oracle (RGO) and initialized at $\mu_0$, outputs a sample from a measure $\mu$ with $\chi^2(\mu \mmid \pi) \le \varepsilon_\chi^2$ using $N$ queries to $\pi$ in expectation, where $N$ satisfies
    \begin{align*}
        N
        &\le \widetilde{\mc O}\Bigl( \CPI \beta d \, \bigl(\log(1+\chi^2(\mu_0 \mmid \pi)) \vee \log \frac{1}{\varepsilon_\chi}\bigr)\Bigr)\,.
    \end{align*}
\end{theorem}

We now briefly justify why this morally leads to an $\mc O(\log(1/\varepsilon))$ complexity guarantee in Fisher information, omitting details for brevity.
Assume that $\beta = 1$ and that $\pi$ is log-concave.
If we set $t\asymp \varepsilon^2/d$ in Lemma~\ref{l:FI-heat-bd}, then we can ensure that $\FI(\mu Q_t \mmid \pi) \le \varepsilon^2$, where $\mu$ is the output of the proximal sampler, provided that $\varepsilon_\chi \le \widetilde{\mc O}(\varepsilon^4/d^2)$. Applying Theorem~\ref{thm:prox_sampler}, this can be achieved using
\begin{align*}
    N = \widetilde{\mc O}\Bigl(\CPI d \, \bigl(\log(1+\chi^2(\mu_0 \mmid \pi)) \vee \log\frac{\sqrt d}{\varepsilon}\bigr)\Bigr)
\end{align*}
queries in expectation. Let us give crude bounds for these terms. First, let $\mf m_2^2 \deq \E_{\pi}[\norm \cdot^2]$ denote the second moment of $\pi$. Then, we know that the Poincar\'e constant of $\pi$ is bounded because $\pi$ is log-concave, and in fact $\CPI \lesssim \mf m_2^2$~\cite[see, e.g.,][]{bobkov1999logconcave}. Also, if $\nabla V(0) = 0$, then we can initialize with $\log(1+\chi^2(\mu_0\mmid \pi)) \le \widetilde{\mc  O}(d)$~\citep[see][Lemma 29]{chewietal2022lmcpoincare}. Putting this together, we see that $N = \poly(d, \mf m_2, \log(\sqrt d/\varepsilon))$ queries suffice in expectation in order to obtain the guarantee $\sqrt{\FI(\mu Q_t \mmid \pi)} \le \varepsilon$. This is in contrast with our lower bound in Theorem~\ref{thm:bump_high_dim}, which shows that $\poly(1/\varepsilon)$ queries are necessary to obtain Fisher information guarantees for \emph{non-log-concave} targets, thereby establishing a separation between log-concave and non-log-concave sampling in this context.

The astute reader will observe that there are some holes in this argument when comparing the lower and upper bounds. Namely, the upper bound uses further properties about the target distribution (e.g., $\nabla V(0) = 0$) and does not strictly hold in the oracle model that we describe in Section~\ref{scn:setting}; the upper bound is in terms of the expected number of queries made, because the number of queries made by the algorithm is random; and the upper bound depends on other parameters such as $\mf m_2$ which do not appear in the lower bound. In particular, the third point requires some consideration because in our lower bound construction for Theorem~\ref{thm:bump_high_dim}, the effective radius $R$ of the distributions  depends on $1/\varepsilon$. We claim, however, that if we set $d, R = \polylog(1/\varepsilon)$, then the upper bound for log-concave targets is $\polylog(1/\varepsilon)$ (with the caveats just discussed) and the lower bound for non-log-concave targets is $\poly(1/\varepsilon)$. As this is not the focus of our work, we do not attempt to make this reasoning more rigorous; rather, we leave it as the sketch of an argument showing that non-log-concave sampling is fundamentally harder than log-concave sampling. We also note that our argument in fact shows that $\polylog(1/\varepsilon)$ query complexity is possible for distributions satisfying a Poincar\'e inequality, which form a strict superclass of log-concave distributions.

\section{Proofs for the first lower bound}

\subsection{Proof of the equivalence}\label{scn:equivalence_pf}

In order to prove the equivalence in Theorem~\ref{thm:equivalence}, we recall the following useful lemma from~\cite{chewietal2022lmcpoincare}.

\begin{lemma}[{\citet[Lemma 16]{chewietal2022lmcpoincare}}]\label{lem:fisher_info_lem}
    Let $\pi\propto\exp(-V)$ be a $\beta$-$\log$-smooth density on $\R^d$. Then, for any probability measure $\mu$,
    \begin{align*}
        \E_\mu[\norm{\nabla V}^2]
        &\le \FI(\mu \mmid \pi) + 2\beta d\,.
    \end{align*}
\end{lemma}

With the lemma in hand, we are ready to prove Theorem~\ref{thm:equivalence}.

\medskip{}

\begin{proof}[Proof of Theorem~\ref{thm:equivalence}]
    \begin{enumerate}
        \item We can explicitly compute
            \begin{align*}
                \FI(\mu_\beta \mmid \pi_\beta)
                &= \int \norm{\nabla \ln \mu_\beta - \nabla \ln \pi_\beta}^2 \, \D \mu_\beta
                = \int \norm{\beta \, (z - x) - \beta\, \nabla V(z)}^2 \, \D \mu_\beta(z) \\
                &\le 2\beta^2 \int \norm{z-x}^2 \, \D \mu_\beta(z) + 2\beta^2 \int \norm{\nabla V(z)}^2 \, \D \mu_\beta(z) \\
                &\le 2\beta^2 \int \norm{z-x}^2 \, \D \mu_\beta(z) + 4\beta^2 \int \{\norm{z-x}^2 +\norm{\nabla V(x)}^2\} \, \D \mu_\beta(z) \\
                &\le 6\beta^2 \int \norm{z-x}^2 \, \D \mu_\beta(z) + 4\beta^2 \,\underbrace{\norm{\nabla V(x)}^2}_{\le \varepsilon^2}\,,
            \end{align*}
            where we used the Lipschitzness of $\nabla V$. Also, $\int \norm{z-x}^2 \, \D \mu_\beta(z) = d/\beta$. Hence,
            \begin{align*}
                \FI(\mu_\beta \mmid \pi_\beta)
                &\le 6\beta d + 4\beta^2 \varepsilon^2
                = 10\beta d\,,
            \end{align*}
            provided $\beta = d/\varepsilon^2$.
        \item Conversely, since $\nabla \ln(1/\pi_\beta) = \beta\, \nabla V$ is $\beta$-Lipschitz, then Lemma~\ref{lem:fisher_info_lem} yields
        \begin{align*}
            \E_\mu[\norm{\nabla V}^2]
            &= \frac{1}{\beta^2} \E_\mu[\norm{\nabla (\beta V)}^2]
            \le \frac{1}{\beta^2} \, \{\FI(\mu \mmid \pi_\beta) + 2\beta d\}
            \le \frac{3d}{\beta}\,.
        \end{align*}
        If we take $\beta = d/\varepsilon^2$, then $\E_\mu[\norm{\nabla V}^2] \le 3\varepsilon^2$.
        By Chebyshev's inequality, $X \sim \mu$ satisfies $\norm{\nabla V(X)} \le \sqrt 6 \,\varepsilon$ with probability at least $1/2$.
    \end{enumerate}
\end{proof}

\subsection{Proof of the averaged LMC guarantee}\label{scn:averaged_lmc_pf}

In order to apply~\citep[Theorem 4]{balasubramanianetal2022nonlogconcave}, we need a bound on the KL divergence at initialization.
Such bounds are standard; 
however, since~\cite[Lemma 30]{chewietal2022lmcpoincare} assumes that we start at a stationary point of $V$ (contrary to the present setting), we present an adapted version.

\begin{lemma}[KL divergence at initialization]\label{lem:init}
    Suppose that $U : \R^d\to\R$ is a function such that $U(0) - \inf U \le \Delta$, $\nabla U$ is $\beta$-Lipschitz, and $\mf m \deq \int \norm \cdot \, \D \pi < \infty$ where $\pi \propto \exp(-U)$.
    Then, for $\mu_0 = \normal(0, \beta^{-1} I_d)$, we have the bound
    \begin{align*}
        \KL(\mu_0 \mmid \pi)
        &\lesssim \Delta + d \, \bigl(1 \vee \ln(\beta \mf m^2)\bigr)\,.
    \end{align*}
\end{lemma}
\begin{proof}
    Write
    \begin{align*}
        \frac{\mu_0}{\pi}
        &= \exp\Bigl( U - \frac{\beta}{2} \, \norm \cdot^2 \Bigr) \, \frac{\int \exp(-U)}{\int\exp(-U-\delta \, \norm \cdot^2)} \, \frac{\int \exp(-U - \delta \, \norm \cdot^2)}{{(2\uppi/\beta)}^{d/2}}\,,
    \end{align*}
    where $\delta > 0$ is chosen later.

    For the first term, by smoothness and Young's inequality,
    \begin{align*}
        U(x) - \frac{\beta}{2} \, \norm x^2
        &\le U(0) + \langle \nabla U(0), x \rangle
        \le U(0) + \frac{\norm{\nabla U(0)}^2}{2\beta} + \frac{\beta \, \norm x^2}{2}\,.
    \end{align*}
    Plugging in $x=-\frac{1}{\beta}\, \nabla U(0)$,
    \begin{align*}
        U\Bigl( - \frac{1}{\beta} \, \nabla U(0) \Bigr) - U(0)
        &\le - \frac{1}{2\beta} \, \norm{\nabla U(0)}^2
    \end{align*}
    or
    \begin{align*}
        \norm{\nabla U(0)}^2
        &\le 2\beta \, \Bigl( U(0) - U\bigl( - \frac{1}{\beta} \, \nabla U(0) \bigr)\Bigr)
        \le 2\beta \, \bigl(U(0) - \inf U\bigr)
        \le 2\beta \Delta\,.
    \end{align*}
    Hence, for any $x$,
    \begin{align*}
        U(x) - \frac{\beta}{2} \, \norm x^2
        &\le U(0) + \Delta + \frac{\beta \, \norm x^2}{2}\,.
    \end{align*}
    For the second term, Markov's inequality yields
    \begin{align*}
        \frac{\int \exp(-U-\delta \, \norm \cdot^2)}{\int \exp(-U)}
        = \int \exp(-\delta \, \norm \cdot^2) \, \D \pi
        &\ge \exp(-4\delta \mf m^2) \, \pi\{\norm \cdot \le 2\mf m\} \\
        &\ge \frac{1}{2} \exp(-4\delta \mf m^2)\,.
    \end{align*}
    For the third term,
    \begin{align*}
        \frac{\int \exp(-U - \delta \, \norm \cdot^2)}{{(2\uppi/\beta)}^{d/2}}
        \le \frac{\exp(-\inf U) \int \exp(-\delta \, \norm \cdot^2)}{{(2\uppi/\beta)}^{d/2}}
        = \exp(-\inf U) \, {\bigl( \frac{\beta}{2\delta} \bigr)}^{d/2}\,.
    \end{align*}
    
    Combining these bounds,
    \begin{align*}
        \KL(\mu_0 \mmid \pi)
        &= \E_{\mu_0} \ln \frac{\mu_0}{\pi}
        \le U(0) - \inf U + \Delta + \frac{\beta}{2} \E_{\mu_0}[\norm \cdot^2] + \ln 2 + 4\delta \mf m^2 + \frac{d}{2} \ln \frac{\beta}{2\delta}\,.
    \end{align*}
    Now we set $\delta = \frac{1}{4\mf m^2}$ to obtain
    \begin{align*}
        \KL(\mu_0 \mmid \pi)
        &\lesssim \Delta + d \, \bigl(1 \vee \ln(\beta \mf m^2)\bigr)
    \end{align*}
    as claimed.
\end{proof}

\begin{proof}[Proof of Corollary~\ref{cor:averaged_lmc}]
    Let $V$ be $1$-smooth and apply the above lemma to $U = \beta V$, which is $\beta$-smooth and satisfies $U(0) -\inf U \le \beta \Delta$, so that
    \begin{align}\label{eq:initial_kl}
        K_0
        &\deq \KL(\mu_0 \mmid \pi_\beta)
        \lesssim \beta \Delta + d \, \bigl(1 \vee \ln(\beta \E_{\pi_\beta}[\norm \cdot^2])\bigr)
        = \widetilde{\mc O}(\beta \Delta + d)\,.
    \end{align}
    The main result of~\citet{balasubramanianetal2022nonlogconcave} says that after $N$ steps of averaged LMC, with an appropriate choice of step size $h$, we output a sample from $\mu$ satisfying
    \begin{align*}
        \FI(\mu \mmid \pi_\beta)
        &\lesssim \frac{\beta \sqrt{K_0 d}}{\sqrt N}\,.
    \end{align*}
    To apply this result, we find $N$ such that this inequality implies $\FI(\mu \mmid \pi_\beta) \le \beta d$, where we recall that $\beta = d/\varepsilon^2$; this requires $N \gtrsim K_0/d$.
    From~\eqref{eq:initial_kl}, it suffices to have $N \ge \widetilde \Omega(\Delta/\varepsilon^2)$, provided $\Delta/\varepsilon^2 \ge 1$.
    The result for finding stationary points via averaged LMC now follows from the equivalence in Theorem~\ref{thm:equivalence}.
\end{proof}

\subsection{Proof of the first lower bound}\label{scn:first_lower_pf}

\begin{proof}[Proof of Theorem~\ref{thm:first_lower}]
    In the lower bound of~\citet{carmonetal2020stationarypt}, the authors construct a family of functions $\ms F$ such that each $f \in \ms F$ is $\beta$-smooth and satisfies $f(0) - \inf f \le \Delta$. Moreover, any randomized algorithm which, for any $f \in \ms F$, makes queries to a local oracle for $f$ and outputs an $\delta$-stationary point of $f$ with probability at least $1/2$, requires at least $\Omega(\beta \Delta/\delta^2)$ queries. The dimension of the functions in the construction is $d = \widetilde \Theta(\beta^2 \Delta^2/\delta^4)$.
    
    We set $\delta \deq 4\sqrt{\beta d}$. From the Fisher information lemma (Lemma~\ref{lem:fisher_info_lem}), if we can obtain a sample from a measure $\mu$ such that for $\pi_f \propto \exp(-f)$, it holds that $\FI(\mu \mmid \pi_f) \le \beta d$, then a sample from $\mu$ is a $\delta$-stationary point of $f$ with probability at least $1/2$.
    
    We set the initialization oracle to simply output samples from $\mu_0 \deq \normal(0, \beta^{-1} I_d)$. We need to compute the value of $K_0 \deq \sup_{f\in\ms F} \KL(\mu_0 \mmid \pi_f)$, and for this we use Lemma~\ref{lem:init}. First, we must bound the second moment $\E_{\pi_f}[\norm \cdot^2]$. Since we only care about polynomial dependencies for this calculation, let $\msf{poly}$ denote any positive quantity for which both the quantity and its inverse are bounded above by polynomials in $\beta$, $\Delta$, $d$, and $1/\delta$. Inspecting the proof of~\citet{carmonetal2020stationarypt} and using the notation therein, each $f \in \ms F$ is of the form
    \begin{align*}
        f(x)
        &= \msf{poly} \cdot \tilde f_{T,U}\bigl(\rho(x/\msf{poly})\bigr) + \frac{1}{2\tau^2} \, \norm x^2\,, \qquad\text{where}~\tau = \msf{poly}\,.
    \end{align*}
    Also, $\tilde f_{T,U}$ is bounded; thus, $\pi_f \propto\exp(-f)$ is well-defined. To bound the second moment of $\pi_f$, we can use the Donsker{--}Varadhan variational principle to write, for any $\lambda > 0$,
    \begin{align*}
        \E_{\pi_f}[\norm \cdot^2]
        &\le \frac{1}{\lambda} \, \{\KL(\pi_f \mmid \nu) + \ln \E_\nu \exp(\lambda \, \norm \cdot^2)\}\,,
    \end{align*}
    where $\nu \deq \normal(0, \tau I_d)$. By choosing $\lambda = 1/\msf{poly}$, we can ensure that $\ln \E_\nu \exp(\lambda \, \norm \cdot^2) \le 1$.
    Next, since $\nu$ satisfies a log-Sobolev inequality with constant $\msf{poly}$, we obtain
    \begin{align*}
        \E_{\pi_f}[\norm \cdot^2]
        &\le \msf{poly} \cdot \bigl(1 + \FI(\pi_f \mmid \nu)\bigr)\,.
    \end{align*}
    The Fisher information is computed to be
    \begin{align*}
        \FI(\pi_f \mmid \nu)
        &= \msf{poly} \cdot \E_{\pi_f}\bigl[\bigl\lVert \nabla\bigl( \tilde f_{T,U}\bigl(\rho(\cdot/\msf{poly})\bigr)\bigr) \bigr\rVert^2\bigr]\,.
    \end{align*}
    Here, $\tilde f_{T,U} : \R^d\to\R$ and $\rho : \R^d\to\R^d$ are $\msf{poly}$-Lipschitz, and hence
    \begin{align*}
        \bigl\lVert \nabla\bigl( \tilde f_{T,U}\bigl(\rho(\cdot/\msf{poly})\bigr)\bigr) \bigr\rVert
        &\le \msf{poly}\,.
    \end{align*}
    Putting everything together, we deduce that $\E_{\pi_f}[\norm \cdot^2] \le \msf{poly}$.
    
    From Lemma~\ref{lem:init}, we can take $K_0 \lesssim \Delta + \widetilde{\mc O}(d)$. If $K_0 \ge \widetilde \Omega(d)$, then this shows that $\Delta \gtrsim K_0$. From the lower bound of~\citet{carmonetal2020stationarypt}, we obtain
    \begin{align*}
        \ms C(d, K_0, \sqrt{\beta d}; \beta)
        &\gtrsim \frac{\beta \Delta}{\delta^2}
        \gtrsim \frac{\beta K_0}{\beta d}
        = \frac{K_0}{d}\,.
    \end{align*}
    
    Finally, in order for the construction of~\citet{carmonetal2020stationarypt} to be valid, the functions must be defined in dimension $d \ge \widetilde \Omega({(K_0/d)}^2)$, which is satisfied provided $d\ge \widetilde \Omega(K_0^{2/3})$.
\end{proof}

\section{Proofs for the second lower bound}\label{sec:pf_lw_bd_2nd}

\subsection{Proof of Theorem~\ref{thm:bump_high_dim}}

Throughout the proof, we will often work with unnormalized densities. For a distribution $\pi$, which we identify with its density, we denote by $\tilde \pi$ an unnormalized density, where $\pi = \frac{\tilde \pi}{Z}$ and the normalizing constant is given by $Z \deq \int_{\R^d} \tilde \pi(x) \, \D x$.

We reduce the task of estimating the distribution from queries to the task of sampling. 
Namely, we construct a set of distributions $\pi$ that are $1$-log-smooth, such that if we can sample well from $\pi$ in Fisher information, then we can estimate its identity. Let $B_r$ denote the ball of radius $r$ in $\R^d$; let $V_d\deq \uppi^{d/2}/\Gamma(\frac d2 + 1)$ denote the volume of $B_1$, and let $A_{d-1}=dV_d$ denote the surface area of $\partial B_1$. Let $\ms P_{2r, R}$ be a maximal $2r$-packing of $B_{R-r}$, for some $R \geq r$ to be specified. By standard volume arguments~\citep[see, e.g.,][\S 4.2]{vershynin2018highdimprob}, we know that $|\ms P_{2r, R}| \ge \bigl(\frac{R-r}{2r}\bigr)^d$. For any $\omega\in \ms P_{2r, R}$, let $\tilde\pi_\omega$ denote the unnormalized density
\begin{equation}\label{eq:pi_unnorm}
    \tilde\pi_\omega(x) \deq \exp\biggl( r^2 \phi\Bigl(\frac{\norm{x-\omega}}{r}\Bigr) - \frac{1}{2}\, {(\norm{x} - R)}_+^2\biggr) \eqqcolon \exp\bigl(-V_\omega(x)\bigr)\,,
\end{equation}
where $(x)_+ \deq \max(0, x)$, and $\phi:\R_+ \to \R_+$ is a bump function with the following properties\footnote{One such function is $\phi(x) = \begin{cases} \frac{11}{64}-x^2\,, &\text{for }x\in[0,1/4]\,,\\ \frac{1}{27}\,\bigl(4+8x-48x^2+56x^3-20x^4\bigr)\,, &\text{for }x\in[1/4,1]\,, \\
0\,,&\text{otherwise}\,.\end{cases}$} :
\begin{enumerate}[label=(\textbf{$\phi$.\arabic*}), ref=(\textbf{$\phi$.\arabic*})]
    \item $\phi$ is continuous, decreasing, supported on $[0,1]$, and twice continuously differentiable on the open interval $(0,\infty)$.\label{phi1}
    \item $\phi(x)=\phi(0)-\frac12\,x^2$ for all $x\in[0,\alpha]$ for some $\alpha > 0$.\label{phi2}
    \item $\sup_{x > 0}{\abs{\phi''(x)}} \le 1$.\label{phi3}
\end{enumerate}
The above implies that on $\R^d$,  $x\mapsto\phi(\norm x)$ is $1$-smooth (see Lemma~\ref{lem:phi_smooth}), and hence $\tilde\pi_\omega$ is $1$-log-smooth. For a measurable set $A$, we will write $\tilde\pi_\omega(A)\deq \int_A \tilde \pi_\omega(x) \,\D x$ and we let $Z_\omega \deq \tilde\pi_\omega(\R^d)$ denote the normalizing constant for $\tilde\pi_\omega$.

We also define the null probability measure $\pinit$ to have unnormalized density
\begin{align*}
    \tpinit(x) \deq \exp\bigl(-\frac12\,{(\norm x-R)}_+^2\bigr)\,,
\end{align*}
with normalizing constant $\Zinit \deq \tpinit(\R^d)$.

The distribution $\pi_\omega$ is the combination of a flat, uniform distribution on $B_R$, fast decaying tails outside of $B_R$, and a bump of radius $r$ around the point $\omega \in \ms P_{2r,R}$. The following lemma summarizes the properties that we need for the lower bound construction. Together, Properties~\ref{p1} and~\ref{p2} imply that if an algorithm outputs a sample $X$ from a distribution which is close in Fisher information to $\pi_\omega$, then $X$ is likely to lie in the set $\omega + B_r$. Hence, an algorithm for sampling from $\pi_\omega$ can be used to \emph{estimate} $\omega$. Property~\ref{p4} is then used to prove a lower bound on the number of queries to solve the estimation task. Finally, Property~\ref{p3} is needed in order to ensure that there is a valid initialization oracle with $K_0 = 1$.

\begin{lemma}[lower bound construction]\label{lem:main}
    There exist universal constants $c_\epsilon, c > 0$ such that for every $d\geq1$ and $\epsilon \leq \exp(-c_\epsilon d)$ we can choose $r$, $R$ such that the following properties hold.
    \begin{enumerate}[label=(\textbf{P.\arabic*}), ref=(\textbf{P.\arabic*})]
        \item (most of the mass lies in the bump) For any $\omega \in \ms P_{2r,R}$,\label{p1}
        \begin{align*}
            \pi_\omega(\omega + B_r) = \frac{1}{2}\,.
        \end{align*}
        \item (FI guarantees imply TV guarantees) For any $\omega \in \mc P_{2r,R}$ and any probability measure $\mu$,\label{p2}
        \begin{align*}
            \sqrt{\FI(\mu \mmid \pi_\omega)} \le \varepsilon \implies \TV(\mu, \pi_\omega) \le \frac{1}{3}\,.
        \end{align*}
        \item (initial KL divergence) There exists a  probability measure $\pinit$ that satisfies\label{p3}
        \begin{align*}
            \max_{\omega \in \ms P_{2r,R}} \KL(\pinit \mmid \pi_\omega) \le \log 2\,.
        \end{align*}
        \item (lower bound on the packing number) It holds that\label{p4}
        \begin{align*}
            \abs{\ms P_{2r,R}}
            &\ge \Bigl(\frac{cd}{\log(1/\epsilon)}\Bigr)^{d/2} \,\frac{1}{\epsilon^{2d/(d+2)}}\,. 
        \end{align*}
    \end{enumerate}
\end{lemma}
\begin{proof}
    Property~\ref{p1} holds by the definition of the parameters $r$ and $R$, see~\eqref{eqn:R r def} and Lemma~\ref{lem:choosing_r_R}. We prove Property~\ref{p2} in Appendix~\ref{scn:pf_p2}, Property~\ref{p3} in Appendix~\ref{scn:pf_p3}, and Property~\ref{p4} in Appendix~\ref{scn:pf_p4}.
\end{proof}

\begin{remark}\label{rem:packing_num_nontrivial}
    In order for the bound in Property~\ref{p4} to be non-trivial, i.e., $\abs{\ms P_{2r,R}} \gtrsim 1$, we require $\varepsilon^{-2d/(d+2)} \gtrsim (\sqrt{\log(1/\varepsilon)/(cd)})^d$. Taking logarithms, we want
    \begin{align*}
        \frac{2d}{d+2}\log \frac{1}{\varepsilon}
        &\overset{!}{\ge} \frac{d}{2}\log\log \frac{1}{\varepsilon} - \frac{d}{2} \log d + \Omega(d)\,.
    \end{align*}
    Let $\gamma$ be such that $\log(1/\varepsilon) = \gamma d$. Substituting this in, we require
    \begin{align*}
        \frac{2\gamma d^2}{d+2}
        &\overset{!}{\ge} \frac{d}{2}\log \gamma + \Omega(d)\,.
    \end{align*}
    This holds as long as $\gamma$ is larger than a universal constant, which is equivalent to $\varepsilon \le \exp(-c_\varepsilon d)$ for a sufficiently large absolute constant $c_\varepsilon > 0$.
\end{remark}

Using the lemma, we can now apply a standard information theoretic argument. We recall the statement of Fano's inequality, see~\cite[\S 2]{coverthomas2006infotheory} for background.

\begin{theorem}[Fano's inequality]\label{thm:fano}
    Let $\omega \sim \unif(\eu X)$, where $\eu X$ is a finite set. Then, for any estimator $\widehat \omega$ which is measurable w.r.t. the data $\xi$, it holds that
    \begin{align*}
        \Pr\{\widehat \omega \ne \omega\}
        &\ge 1 - \frac{I(\xi; \omega) + \ln 2}{\ln{\abs{\eu X}}}\,,
    \end{align*}
    where $I$ denotes the mutual information.
\end{theorem}

With this theorem in hand, we are ready to prove Theorem~\ref{thm:bump_high_dim}.

\medskip{}

\begin{proof}[Proof of Theorem~\ref{thm:bump_high_dim}]
    Let $\omega \sim \unif(\ms P_{2r,R})$ and consider the task of estimating $\omega$ with randomized algorithms that have query access to $\pi_\omega$. We first show that a sampling algorithm can solve this estimation task. Suppose that there is an algorithm that works under the oracle model specified in Section~\ref{scn:setting}, with initialization oracle outputting samples from $\mu_0 = \pinit$ given in Property~\ref{p3}, which guarantees that $\KL(\mu_0\mmid \pi_\omega) \le \log 2$. For any $\omega \in \ms P_{2r,R}$ and target $\pi_\omega$, the algorithm makes at most $N$ queries to the local oracle, and outputs a sample from the measure $\mu_N$ with $\sqrt{\FI(\mu_N \mmid \pi_\omega)} \le \varepsilon$. We can then estimate $\omega$ as follows: let $X \sim \mu_N$, and set
    \begin{align*}
        \widehat \omega
        &\deq \argmin_{\omega \in \ms P_{2r,R}}{\norm{X - \omega}}\,.
    \end{align*}
    Because the initialization oracle $\mu_0$ is independent of the choice of $\omega$, the estimator $\widehat \omega$ is the output of a randomized algorithm that only uses the query information to $\pi_\omega$ to estimate $\omega$.
    
    The probability that $\widehat \omega$ succeeds can be calculated as follows. By Property~\ref{p2}, we have $\TV(\mu_N, \pi_\omega) \le 1/3$. Let $X \sim \mu_N$; then,
    \begin{align*}
        \Pr\{X \in \omega + B_r\}
        &= \mu_N(\omega + B_r)
        \ge \pi_\omega(\omega + B_r) - \TV(\mu_N, \pi_\omega)
        \ge \frac{1}{2} - \frac{1}{3} = \frac{1}{6}\,,
    \end{align*}
    where we used Property~\ref{p1}. Hence
    we see that
    \begin{align}\label{eq:estimator_success}
        \Pr\{\widehat \omega = \omega\}
        &\ge \frac{1}{6}\,.
    \end{align}
    
    Now we prove a lower bound for the estimation task for any algorithm that succeeds with probability at least $\frac 16$. Let $x_1,\dotsc,x_N$ denote the query points made by the algorithm.  We first prove a lower bound for deterministic algorithms, where each query point $x_i$ is a deterministic function of the previous queries and oracle outputs $(x_{i'}, V_\omega(x_{i'}), \nabla V_\omega(x_{i'}) : i'=1,\dotsc,i-1)$. Since the initialization oracle is independent of $\omega$, the data available to the algorithm is
    \begin{align*}
        \xi_N
        &\deq \bigl(x_i, V_\omega(x_i), \nabla V_\omega(x_i) : i=1,\dotsc,N\bigr)\,.
    \end{align*}
    We assume that the algorithm has made at most $N \le M/2$ queries where $M \deq \abs{\ms P_{2r,R}}$ (otherwise, $N \ge M/2$ and this is our desired lower bound).
    
    Applying Fano's inequality (Theorem~\ref{thm:fano}):, we therefore have
    \begin{align}\label{eq:apply_fano}
        \Pr\{\widehat \omega \ne \omega\}
        &\ge 1 - \frac{I(\xi_N; \omega) + \ln 2}{\ln M}.
    \end{align}
    Applying the chain rule for the mutual information,
    \begin{align*}
        I(\xi_N; \omega)
        &\le \sum_{i=1}^N I\bigl(x_i, V_\omega(x_i), \nabla V_\omega(x_i);\; \omega \bigm\vert \xi_{i-1}\bigr)\,.
    \end{align*}
    Given $\xi_{i-1}$, the query point $x_i$ is deterministic. We can bound the mutual information via the conditional entropy,
    \begin{align*}
        I\bigl(x_i, V_\omega(x_i), \nabla V_\omega(x_i);\; \omega \bigm\vert \xi_{i-1}\bigr)
        &\le H\bigl(V_\omega(x_i), \nabla V_\omega(x_i) \bigm\vert \xi_{i-1}\bigr)\,.
    \end{align*}
    
    If one of the query points $x_1,\dotsc,x_{i-1}$ landed in the ball $\omega + B_r$, then $\omega$ is fully known and the conditional entropy is zero. Otherwise, given the history $\xi_{i-1}$, the random variable $\omega$ is uniformly distributed on the set
    \begin{align*}
        \ms P_{2r,R}(i)
        &\deq \{\omega' \in \ms P_{2r,R} \mid x_{i'}\notin \omega' + B_r~\text{for}~i=1,\dotsc,i-1\}\,.
    \end{align*}
    If $x_i$ does not belong to $\omega' + B_r$ for some $\omega' \in \ms P_{2r,R}(i)$, then the query point is useless and the conditional entropy is again zero.
    Otherwise, conditionally on $\xi_{i-1}$, the tuple $(V_\omega(x_i), \nabla V_\omega(x_i))$ can take on two possible values with probability $1/\abs{\ms P_{2r,R}(i)}$ and $1-1/\abs{\ms P_{2r,R}(i)}$ respectively, depending on whether or not $x_i \in \omega + B_r$. The conditional entropy is thus bounded by
    \begin{align*}
        H\bigl(V_\omega(x_i), \nabla V_\omega(x_i) \bigm\vert \xi_{i-1}\bigr)
        &\le h\Bigl( \frac{1}{\abs{\ms P_{2r,R}(i)}} \Bigr)
        \le h\bigl( \frac{2}{M} \bigr)\,,
    \end{align*}
    where $h(p) \deq p\ln \frac{1}{p} + (1-p) \ln \frac{1}{1-p}$ is the binary entropy function. Assuming that $M\ge 4$ (which can be ensured thanks to Remark~\ref{rem:packing_num_nontrivial}),
    \begin{align*}
        h\bigl(\frac{2}{M}\bigr)
        &\le \frac{4}{M} \ln \frac{M}{2}\,.
    \end{align*}
    
    Substituting this into~\eqref{eq:apply_fano},
    \begin{align}\label{eq:est_lw}
        \Pr\{\widehat \omega \ne \omega\}
        &\ge 1 - \frac{\frac{4N}{M} \ln(M/2) + \ln 2}{\ln M}\,.
    \end{align}
    If $M \ge 4$, and $N \le \frac{1}{12}\,M$, we would obtain $\Pr\{\widehat \omega \ne \omega\} > 5/6$, contradicting~\eqref{eq:estimator_success}. Hence, we deduce that $N \gtrsim M$.
    
    In general, if the algorithm is randomized, it can depend on a random seed $\zeta$ that is independent of $\omega$. Then we can apply \eqref{eq:est_lw} conditional on $\zeta$, and obtain
    \begin{align*}
        \Pr\{\widehat \omega \ne \omega \mid \zeta\}
        &\ge 1 - \frac{\frac{4N}{M} \ln(M/2) + \ln 2}{\ln M}\,.
    \end{align*}
    Taking expectation over $\zeta$, we see that the lower bound \eqref{eq:est_lw}, and hence $N\gtrsim M$, holds for randomized algorithms as well.
    
    The proof of Theorem~\ref{thm:bump_high_dim} is concluded by noting that the estimation lower bound gives a lower bound on sampling, and that Property~\ref{p4} provides us with a lower bound on $M$.
\end{proof}

In the remaining sections, we focus on establishing Lemma~\ref{lem:main}.

\subsection{Estimates for integrals}

In this section we provide useful estimates for integrals that appear in the normalizing constants for our lower bound construction.
Notice that since $\tilde \pi_\omega = 1$ on $B_R \setminus (\omega + B_r)$,
\begin{align*}
    Z_\omega
    = \tilde \pi_\omega(\R^d\setminus B_R) + \tilde \pi_\omega(B_R)
    &= \tilde \pi_\omega(\R^d\setminus B_R) + (R^d-r^d)\, V_d + \int_{B_r} \exp\Bigl(r^2 \phi\bigl(\frac{\norm x}{r}\bigr) \Bigr) \, \D x \\
    &= \tilde \pi_\omega(\R^d\setminus B_R) + (R^d-r^d)\, V_d + r^d \, I_r\,,
\end{align*}
where we define $I_r \deq \int_{B_1}\exp(r^2\phi(\norm x))\,\D x$. We record some useful properties of the quantities defined thus far that will be used throughout the proof of Lemma~\ref{lem:main}.

\begin{lemma}[main estimates] \label{lem:high d properties}
For any number $c > 0$ there exists $c_r(c) > 0$ depending only on $c$ such that for all $r \geq c_r(c) \sqrt{d}$, the following hold:
\begin{enumerate}
    \item (asymptotics of $I_r$)
    \begin{align}\label{eqn:I_r bound}
        \frac{1}{2} \le \frac{r^d \,I_r}{(2\uppi)^{d/2} \exp(r^2\phi(0))} \le 2\,.
    \end{align}
    \item (mass outside $B_R$) There is a universal constant $\czero > 2$, independent of $c$, such that
    \begin{align}\label{eqn:pi_w bound}
        \sqrt{\frac{\uppi}{2}} \,V_d d R^{d-1} &\leq \tilde\pi_\omega(\R^d \setminus B_R) \leq V_d  \czero^d \,(dR^{d-1} + d^{(d+1)/2})\,.
    \end{align}
    \item (mass on the bump)
    \begin{align}\label{eqn:log I/V}
    \ln \frac{I_r}{V_d} &\geq cd\,.
\end{align}
\end{enumerate}
\end{lemma}

\begin{proof}
Because we have chosen $\phi$ to be a quadratic function on the range $[0,\alpha]$ (see~\ref{phi2}), we can decompose $I_r$ as follows:
\begin{align*}
I_r &\deq \int_{B_1} \exp\bigl(r^2 \phi(\norm{x})\bigr)\, \D x\\
&= {\underbrace{\int_{B_1\setminus B_{\alpha }} \exp\bigl(r^2\phi(\norm x)\bigr)\,\D x}_{\mathsf{A}}} + {\underbrace{\exp\bigl(r^2\phi(0)\bigr) \int_{B_{\alpha }} \exp\Bigl(-\frac{r^2\,\norm x^2}{2}\Bigr)\, \D x}_{\mathsf{B}}}\, .
\end{align*}
As $\phi$ is decreasing by~\ref{phi1}, clearly $\msf{A}\leq V_d\exp(r^2\phi(\alpha))$, and the second term is given by
\begin{align*}
\mathsf{B} &= \frac{(2\uppi)^{d/2}}{r^d} \exp\bigl(r^2\phi(0)\bigr)\, \Pr(\norm X \leq \alpha r)\,,
\end{align*}
where $X$ is a standard Gaussian in $\R^d$.
By standard concentration inequalities (e.g., Markov's inequality suffices), there exists a universal constant $c_1$ such that the above probability is at least $1/2$ provided $r\geq c_1\sqrt{d}$. Recall that $\log\Gamma(\frac d2 + 1) = \frac d2 \log d + \mathcal O(d)$. Thus, for $r \geq c_1 \sqrt{d}$ we have
\begin{align*}
\log\frac{\msf{A}}{\msf{B}} &\leq \log\frac{2 V_d \exp(r^2\phi(\alpha))}{(2\uppi)^{d/2}\,r^{-d}\exp(r^2\phi(0))}
= \log\frac{2 \exp(r^2\phi(\alpha))}{2^{d/2}\,r^{-d}\exp(r^2\phi(0))\,\Gamma(\frac{d}{2}+1)} \\
&= \mathcal O(d) - r^2\,(\phi(0)-\phi(\alpha)) + d\log r - \frac d2 \log d \\
&= \mathcal O(d) - d\,\Bigl(\bigl(\frac{r}{\sqrt{d}}\bigr)^2\,(\phi(0)-\phi(\alpha))-\log\big( \frac{r}{\sqrt{d}}\big)\Bigr)\,.
\end{align*}
From the above it is clear that there is a universal constant $c_2$ such that $r \geq c_2\sqrt{d}$ implies that $\msf A \le \msf B$. Thus, for $r \geq (c_1\lor c_2)\sqrt{d}$ the following holds:
\begin{equation}
\msf{B} \leq I_r \leq 2\msf{B}\,,
\end{equation}
proving \eqref{eqn:I_r bound}. We now turn to the proof of \eqref{eqn:pi_w bound}. By integrating in polar coordinates, and taking $X$ to be a standard Gaussian on $\R$,
\begin{align*}
\tilde \pi_\omega(\R^d\setminus B_R)
&= A_{d-1} \int_R^\infty  s^{d-1} \exp\bigl(-\frac12\,{(s-R)}^2\bigr)\, \D s \\
&= A_{d-1} \int_0^\infty  {(s+R)}^{d-1} \exp\bigl(-\frac{s^2}{2}\bigr)\, \D s \\
&\leq \sqrt{2\uppi}\, A_{d-1} \E[|X+R|^{d-1}] \\
&\leq  \sqrt{2\uppi}\, A_{d-1} 2^d \,(R^{d-1}+\E[|X|^{d-1}]) \\
&\leq A_{d-1} \czero^d \,(R^{d-1} + (d-1)^{(d-1)/2}) \\
&\leq V_d \czero^d \,(dR^{d-1} + d^{(d+1)/2})
\end{align*}
for some universal constant $\czero > 2$. For the other direction we can simply write
\begin{align*}
\tilde \pi_\omega(\R^d\setminus B_R) &= A_{d-1} \int_0^\infty  {(s+R)}^{d-1} \exp\bigl(-\frac{s^2}{2}\bigr)\, \D s \\
&\geq \sqrt{\frac{\uppi}{2}}\, A_{d-1} R^{d-1}\,. 
\end{align*}

Finally, we prove~\eqref{eqn:log I/V}. We again use the fact $\log \Gamma(\frac d2 + 1) = \frac d2 \log d + \mathcal O(d)$. Therefore, for $r \geq (c_1\lor c_2)\sqrt{d}$ and using \eqref{eqn:I_r bound} we obtain
\begin{align*}
\log\frac{I_r}{V_d}
&\geq \log\frac{(2\uppi)^{d/2}\,r^{-d}\exp(r^2\phi(0))/2}{\uppi^{d/2}/\Gamma(\frac d2 + 1)} \\
&= d\,\Bigl(\bigl(\frac{r}{\sqrt{d}}\bigr)^2 \,\phi(0)-\log\big(\frac{r}{\sqrt{d}}\big)\Bigr) + \mathcal O(d)\,.
\end{align*}
Clearly, there exists a constant $c_3$ (depending only on $c$) such that $r \geq c_3 \sqrt{d}$ implies that the RHS is at least linear in $d$ with a positive constant. Taking $c_r = c_1\lor c_2\lor c_3$ concludes the proof. 
\end{proof}

\subsection{Proof of Property~\ref{p1}}

We choose $r$, $R$ such that~\ref{p1} holds, i.e., $\pi_\omega(\omega + B_r) = 1/2$.
This holds provided that
\begin{equation}\label{eqn:R r def}
    f(r)\deq (I_r+V_d)\, r^d \overset{!}{=} \tilde\pi_\omega(\R^d\setminus B_R) + V_dR^d \eqqcolon g(R)\,. 
\end{equation}

\begin{lemma}[choice of $r$, $R$]\label{lem:choosing_r_R}
    For any value of $d\ge 1$ and $R\ge 0$, there exists a corresponding value of $r$ such that~\eqref{eqn:R r def} holds. Moreover, there is a universal constant $c_R \ge 1$ such that for any $R\ge c_R\sqrt d$, the corresponding $r$ solving~\eqref{eqn:R r def} satisfies
    \begin{align}
    r &\geq c_r\bigl(\log(6 \czero)\bigr) \sqrt{d}\,, \label{eqn:r>sqrt d} \\
        R/r &\geq 2\label{eqn:R/r>2}\,,
    \end{align}
    where $c_r(\cdot)$ is the function defined in Lemma~\ref{lem:high d properties}.
\end{lemma}

The argument $\log(6 \czero)$ to $c_r(\cdot)$ in Lemma~\ref{lem:choosing_r_R} is chosen for later convenience.

\medskip{}

\begin{proof}
 Notice that $f$ and $g$ are continuous and increasing in $r,R$ respectively. Moreover, we check that $f(0)=0$, $g(0)=(2\uppi)^{d/2}$, and $f(\infty)=g(\infty)=\infty$. This tells us that for any value of $d\ge 1$ and $R \geq 0$, there exists a value of $r\geq0$ for which $f(r)=g(R)$.
 
 For the rest of the proof, we abbreviate $c_r \deq c_r(\log(6\czero))$.
 
 First, we prove~\eqref{eqn:r>sqrt d}. Note that since~\eqref{eqn:r>sqrt d} is a hypothesis of Lemma~\ref{lem:high d properties}, we cannot invoke Lemma~\ref{lem:high d properties} during the proof of~\eqref{eqn:r>sqrt d} in order to avoid a circular argument.
 
 By the definitions of $r$ and $R$,
 \begin{align*}
     (I_r + V_d)\,r^d
     &\ge V_d \,R^d\,.
 \end{align*}
 Taking logarithms and using the definition of $I_r$, this rewrites as
 \begin{align*}
     d\log \frac{R}{r}
     &\le \log\bigl(1 + \frac{I_r}{V_d}\bigr)
     = \log\Bigl(1 + \frac{\int_{B_1} \exp(r^2 \phi(\norm x)) \, \D x}{V_d}\Bigr)
     \le \log\bigl(1 + \exp\bigl(r^2 \phi(0)\bigr)\bigr)\,.
 \end{align*}
 Suppose, for the sake of contradiction, that $r < c_r \sqrt d$. Then, we have
 \begin{align*}
     d\log \frac{R}{r}
     &\le c_r^2 d \phi(0) + \log 2\,.
 \end{align*}
 Rearranging,
 \begin{align*}
     R
     &\le \exp\Bigl(c_r^2 \phi(0) + \frac{\log 2}{d}\Bigr) \, r
     \le \exp\Bigl(c_r^2 \phi(0) + \frac{\log 2}{d}\Bigr) \, c_r \sqrt d\,.
 \end{align*}
 Hence, if $R \ge c_R\sqrt d$ for a large enough universal constant $c_R$, then we arrive at the desired contradiction. For later convenience we choose $c_R$ to always be at least $1$. This proves~\eqref{eqn:r>sqrt d}.
 
 Next, we prove~\eqref{eqn:R/r>2}. We use the fact that $R \ge c_R\sqrt d$; so that in particular $c_R \ge 1$ and thus $\sqrt d \le R$. Then, using~\eqref{eqn:pi_w bound} from Lemma~\ref{lem:high d properties},
 \begin{align*}
     (I_r + V_d) \, r^d
     &\le V_d \, (\czero^d d R^{d-1} + \czero^d d^{(d+1)/2} + R^d)
     \le V_d \, (\czero^d \sqrt d R^d + \czero^d \sqrt d R^d + R^d) \\
     &\le V_d \cdot 3\czero^d \sqrt d\, R^d\,.
 \end{align*}
 Taking logarithms, rearranging, and using~\eqref{eqn:log I/V} from Lemma~\ref{lem:high d properties},
 \begin{align*}
     d\log \frac{R}{r}
     &\ge \log\bigl(1 + \frac{I_r}{V_d}\bigr) - d\log \czero - \log(3\sqrt d)
     \ge \bigl(c - \log \czero - \underbrace{\frac{\log(3\sqrt d)}{d}}_{\le \log 3} \bigr) \, d\,.
 \end{align*}
 Taking $c = \log \czero + \log 3 + \log 2 = \log(6\czero)$, this implies $R/r \ge 2$ as desired.
\end{proof}

\subsection{Proof of Property~\ref{p2}}\label{scn:pf_p2}

The proof of Property~\ref{p2} requires an upper bound on the Poincar\'e constant of $\pi_\omega$. We recall that the Poincar\'e constant of a probability measure $\pi$ is the smallest constant $\CPI(\pi) > 0$ such that for all smooth and bounded test functions $f : \R^d\to\R$, it holds that
\begin{align*}
    \var_\pi(f)
    &\le \CPI(\pi)\E_\pi[\norm{\nabla f}^2]\,.
\end{align*}
We begin with a Poincar\'e inequality for $\pinit$.

\begin{lemma}[Poincar\'e inequality for $\pinit$]\label{lem:poincare const}
    If $R\ge \sqrt d$, then the probability measure $\pinit$ has Poincar\'e constant at most $\cPI R^2/d$ for a universal constant $\cPI$. 
\end{lemma}
\begin{proof}
    From~\citet{bob2003radially} and the fact that $\pinit$ is a radially symmetric log-concave measure, the Poincar\'e constant of $\pinit$ is bounded by
    \begin{align*}
        \CPI(\pinit)
        &\le \frac{13\E_{\pinit}[\norm\cdot^2]}{d}\,.
    \end{align*}
    The second moment is
    \begin{align*}
        \E_{\pinit}[\norm\cdot^2]
        &= \frac{\int_{B_R} \norm\cdot^2}{\Zinit} + \frac{\int_{\R^d\setminus B_R} \norm \cdot^2 \exp(-\frac{1}{2} \, {(\norm \cdot - R)}^2)}{\Zinit} \\
        &\le \frac{\int_{B_R} \norm\cdot^2}{V_d R^d} + \frac{A_{d-1} \int_0^\infty {(r+R)}^{d+1} \exp(-r^2/2) \, \D r}{A_{d-1} \int_0^\infty {(r+R)}^{d-1} \exp(-r^2/2) \, \D r}
        \le R^2 + \int {(r+R)}^2 \,\nu(\D r)\,,
    \end{align*}
    where $\nu$ is the probability measure on $\R_+$ with density
    \begin{align}\label{eq:tilted_density}
        \nu(r)
        &\propto {(r+R)}^{d-1} \exp\bigl( - \frac{r^2}{2}\bigr)\,.
    \end{align}
    Note that $\nu$ is $1$-strongly-log-concave.
    Hence, by~\citet[Proposition 1]{durmus2019high},
    \begin{align*}
        \int {(r+R)}^2 \,\nu(\D r)
        &\lesssim R^2 + \int r^2 \, \nu(\D r)
        \lesssim R^2 + r_\star^2 + \int {(r-r_\star)}^2 \, \nu(\D r)
        \le R^2 + r_\star^2 + 1\,,
    \end{align*}
    where $r_\star$ is the mode of $\nu$. To find the mode,~\eqref{eq:tilted_density} and elementary calculus show that $r_\star$ satisfies $r_\star \, (r_\star + R) = d-1$, which implies $r_\star \le (d-1)/R$. If $R\ge \sqrt d$, then $r_\star \lesssim R$. Combining the bounds, we obtain $\CPI(\pinit) \lesssim R^2/d$.
\end{proof}

Next, we recall the statement of the Holley{--}Stroock perturbation principle.

\begin{theorem}[{Holley--Stroock perturbation principle,~\citet{holley1986logarithmic}}]\label{thm:holley-stroock}\\
    Let $\pi$ be a probability measure which satisfies a Poincar\'e inequality. Suppose that $\mu$ is another probability measure such that
    \begin{align*}
        0
        &< c \le \frac{\D \mu}{\D \pi}
        \le C < \infty\,.
    \end{align*}
    Then, $\mu$ also satisfies a Poincar\'e inequality, with
    \begin{align*}
        \CPI(\mu)
        &\le \frac{C}{c} \, \CPI(\pi)\,.
    \end{align*}
\end{theorem}
\begin{proof}
    See~\citet[Lemma 5.1.7]{bakrygentilledoux2014}.
\end{proof}

\begin{corollary}[Poincar\'e inequality for $\pi_\omega$]\label{cor:poincare_pi_w}
    Assume that $R\ge \sqrt d$.
    Then, for each $\omega \in \ms P_{2r,R}$,
    \begin{align*}
        \CPI(\pi_\omega)
        &\le \frac{2\cPI R^2}{d} \exp\bigl(r^2 \phi(0)\bigr)\,.
    \end{align*}
\end{corollary}
\begin{proof}
    By~\ref{phi1}, we know that $\tilde \pi_\omega \ge \tpinit$ and hence $Z_\omega \ge \Zinit$. It follows that
    \begin{align*}
        \frac{\Zinit}{Z_\omega}
        &\le \frac{\pi_\omega}{\pinit}
        = \frac{\tilde \pi_\omega}{\tpinit} \, \frac{\Zinit}{Z_\omega}
        \le \frac{\tilde \pi_\omega}{\tpinit}
        \le \exp\bigl(r^2 \phi(0)\bigr)\,.
    \end{align*}
    Also, by~\eqref{eqn:R r def},
    \begin{align*}
        Z_\omega
        = \tilde \pi_\omega(\R^d \setminus B_R) + V_d\, R^d + (I_r - V_d)\, r^d
        &\le \tilde \pi_\omega(\R^d \setminus B_R) + V_d\, R^d + (I_r + V_d) \, r^d \\
        &= 2 \, \bigl(\tilde \pi_\omega(\R^d \setminus B_R) + V_d\, R^d\bigr)
        = 2\,\Zinit\,.
    \end{align*}
    Hence, $\Zinit/Z_\omega \ge 1/2$. The result now follows from Lemma~\ref{lem:poincare const} and the Holley{--}Stroock perturbation principle (Theorem~\ref{thm:holley-stroock}).
\end{proof}

To translate Fisher information guarantees into total variation guarantees, we use the following consequence of the Poincar\'e inequality.

\begin{proposition}[Fisher information controls total variation]\label{prop:fi_transport}
    Suppose that a probability measure $\pi$ satisfies a Poincar\'e inequality.
    Then, for any probability measure $\mu$,
    \begin{align*}
        \TV(\mu,\pi)^2
        &\le \frac{\CPI(\pi)}{4} \FI(\mu \mmid \pi)\,.
    \end{align*}
\end{proposition}
\begin{proof}
    See~\citet{guillinetal2009transport}.
\end{proof}

We are finally ready to prove Property~\ref{p2}. More specifically, we will show that there is a universal constant $c_\varepsilon > 0$ such that if $\varepsilon \le \exp(-c_\varepsilon d)$, then we can choose $r$ and $R$ (depending on $\varepsilon$) such that: (i) $r$ and $R$ are related according to~\eqref{eqn:R r def}, which is necessary for Property~\ref{p1}; (ii) $R \ge c_R\sqrt d$, which is necessary for Lemma~\ref{lem:choosing_r_R}; and (iii) Property~\ref{p2} holds.

\medskip{}

\begin{proof}[Proof of Property~\ref{p2}]
For any $\omega \in \ms P_{2r, R}$, suppose that $\mu$ satisfies $\sqrt{\FI(\mu\mmid \pi_\omega)} \leq \epsilon$. Then, by Corollary~\ref{cor:poincare_pi_w} and Proposition~\ref{prop:fi_transport}, we have
\begin{align}\label{eq:p2_proof}
    \TV^2(\mu, \pi_\omega)
    &\leq \frac{\CPI(\pi_\omega)}{4} \FI(\mu \mmid \pi)
    \le \frac{\cPI R^2 \exp(r^2 \phi(0))}{2d} \,\varepsilon^2\,.
\end{align}
Hence, if we choose
\begin{align}\label{eq:choose_R}
    R^2
    &= \frac{2d}{9\cPI \varepsilon^2 \exp(r^2\phi(0))}
\end{align}
then $\sqrt{\FI(\mu\mmid\pi_\omega)} \le \varepsilon$ implies $\TV(\mu,\pi_\omega) \le 1/3$, i.e., Property~\ref{p2} holds.

To justify~\eqref{eq:choose_R}, note that thus far we have shown that for any choice of $R$, there exists a choice of $r$ which depends on $R$, which we temporarily denote by $r(R)$, such that~\eqref{eqn:R r def} holds. Also, $r(\cdot)$ is an increasing function. In order for~\eqref{eq:choose_R} to hold, it is equivalent to require
\begin{align}\label{eq:choose_R_2}
    R^2 \exp\bigl({r(R)}^2 \,\phi(0)\bigr)
    &= \frac{2d}{9\cPI \varepsilon^2}
\end{align}
where the left-hand side is an increasing function of $R$. We also want $R$ to satisfy $R \ge c_R\sqrt d$, where $c_R$ is the universal constant in Lemma~\ref{lem:choosing_r_R}. From Lemma~\ref{lem:choosing_r_R}, for the choice of $R = c_R\sqrt d$,
\begin{align*}
    r(c_R \sqrt d) \le \frac{c_R \sqrt d}{2}\,.
\end{align*}
Therefore, for this choice of $R$, the left-hand side of~\eqref{eq:choose_R_2} is bounded by
\begin{align*}
    c_R^2 d \exp\bigl( \frac{c_R^2 d}{4}\, \phi(0)\bigr)\,.
\end{align*}
If it holds that
\begin{align}\label{eq:eps_cond}
    \varepsilon^2
    &\le \frac{2}{9\cPI c_R^2 \exp(c_R^2 d\phi(0)/4)}
\end{align}
then the $R$ satisfying~\eqref{eq:choose_R} necessarily satisfies $R \ge c_R\sqrt d$. In turn,~\eqref{eq:eps_cond} holds if $\varepsilon\le\exp(-c_\varepsilon d)$ for a universal constant $c_\varepsilon > 0$.
\end{proof}

\subsection{Proof of Property~\ref{p3}}\label{scn:pf_p3}

\begin{proof}[Proof of Property~\ref{p3}]
    In the proof of Corollary~\ref{cor:poincare_pi_w}, we showed that $Z_\omega \le 2\,\Zinit$. The KL divergence is bounded by
    \begin{align*}
        \KL(\pinit \mmid \pi_\omega)
        &= \E_{\pinit} \ln\bigl( \underbrace{\frac{\tpinit}{\tilde\pi_\omega}}_{\le 1} \, \underbrace{\frac{Z_\omega}{\Zinit}}_{\le 2}\bigr)
        \le \log 2\,,
    \end{align*}
    which is what we wanted to show.
\end{proof}

\subsection{Proof of Property~\ref{p4}}\label{scn:pf_p4}

\begin{proof}[Proof of Property~\ref{p4}]
    We choose $r$ and $R$ to satisfy~\eqref{eqn:R r def} and~\eqref{eq:choose_R}.
    If $\varepsilon \le \exp(-c_\varepsilon d)$, then we showed in the proof of Property~\ref{p2} that $R \ge c_R\sqrt d$ and hence Lemmas~\ref{lem:high d properties} and~\ref{lem:choosing_r_R} apply.
    
As in the proof of~\eqref{eqn:R/r>2} in Lemma~\ref{lem:choosing_r_R}, $R\ge \sqrt d$ implies
\begin{align*}
    (I_r + V_d) \, r^d
    &\le V_d \cdot 3\czero^d \sqrt d \, R^d\,.
\end{align*}
Taking logarithms in~\eqref{eqn:I_r bound} from Lemma~\ref{lem:high d properties} and using the above inequality, we obtain
\begin{align*}
    r^2 \phi(0)
    &\le \log \frac{2r^d I_r}{{(2\uppi)}^{d/2}}
    \le \mc O(d) + \log V_d + d\log R\,.
\end{align*}
From~\eqref{eq:choose_R}, we have
\begin{align*}
    \log R
    &= \frac{1}{2} \log d + \log \frac{1}{\varepsilon} - \frac{1}{2}\,r^2 \phi(0) + \mc O(1)\,.
\end{align*}
Substituting this in and using $\log V_d = -\frac{d}{2} \log d + \mc O(d)$,
\begin{align*}
    r^2 \phi(0)
    &\le d\log\frac{1}{\varepsilon} - \frac{d}{2} \,r^2 \phi(0) + \mc O(d)
\end{align*}
which is rearranged to yield
\begin{align*}
    r^2 \phi(0)
    &\le \frac{2d}{d+2} \log\frac{1}{\varepsilon} + \mc O(1)\,.
\end{align*}

Then, the packing number is lower bounded by
\begin{align*}
   |\ms P_{2r, R}| \ge  \Bigl(\frac{R-r}{2r}\Bigr)^d &\geq \Bigl(\frac{R}{4r}\Bigr)^d \\
    &\geq \Bigl(c\,\frac{\sqrt{d} \exp(-\frac{d}{d+2}\log(1/\epsilon))}{\epsilon \sqrt{\log(1/\epsilon)}}\Bigr)^d \\
    &\geq \Bigl(c \,\sqrt{\frac{d}{\log(1/\epsilon)}}\Bigr)^d\, \frac{1}{\epsilon^{2d/(d+2)}}\, ,
\end{align*}
for some universal constant $c$.
\end{proof}

\subsection{Auxiliary lemmas}

\begin{lemma}\label{lem:phi_smooth}
    Suppose that $\phi : \R_+\to\R_+$ satisfies~\ref{phi1},~\ref{phi2}, and~\ref{phi3}.
    Then, the map $x\mapsto \phi(\norm x)$ is $1$-smooth on $\R^d$.
\end{lemma}
\begin{proof}
First, we claim that $\abs{\phi'(x)}/x \le 1$ for all $x > 0$. This follows from~\ref{phi3} because~\ref{phi2} implies that the right derivative $\phi'(0+)$ exists and equals $0$.

Next, we have for $x\ne 0$
\begin{align*}
    \frac{\partial^2 \phi(\norm x)}{\partial x_j\,\partial x_i} &= \frac{\partial}{\partial x_j}\, \phi'(\norm x)\,\frac{x_i}{\norm x} \\
    &= \phi''(\norm x)\,\frac{x_ix_j}{\norm x^2} - \phi'(\norm x)\,\frac{x_ix_j}{\norm x^3} + \delta_{i,j}\, \phi'(\norm x)\,\frac{1}{\norm x}\,.
\end{align*}
Thus, in matrix form we have
\begin{equation*}
    \nabla^2_x \phi(\norm x) = \frac{\phi'(\norm x)}{\norm x}\, I_d + \Bigl(\frac{\phi''(\norm x)}{\norm x^2}-\frac{\phi'(\norm x)}{\norm x^3}\Bigr)\, xx^\T. 
\end{equation*}
In particular, the eigenvalues are always $\frac{\phi'(\norm x)}{\norm x}$ with multiplicity $d-1$ and $\phi''(\norm x)$ with multiplicity $1$. The fact that $\phi(\norm\cdot)$ is $1$-smooth follows.
\end{proof}

\subsection{Optimization of the bound}\label{scn:optimize_bd}

We wish to find $d$ which maximizes
\begin{align*}
    \Bigl( \frac{cd}{\log(1/\varepsilon)} \Bigr)^{d/2} \, \varepsilon^{4/(d+2)}\,,
\end{align*}
or after taking logarithms, we wish to maximize
\begin{align*}
     f(d)
     &\deq \frac{d}{2} \log d - \frac{4}{d+2} \log \frac{1}{\varepsilon} - \frac{d}{2} \log \log\frac{1}{\varepsilon} - \frac{d}{2} \log \frac{1}{c}\,.
\end{align*}
Rather than maximizing this expression exactly, we shall ignore the last two terms and pick $d$ to be the smallest integer such that the sum of the first two terms is non-negative, i.e.,
\begin{align*}
    \frac{d \, (d+2) \log d}{8}
    &\ge \log \frac{1}{\varepsilon}\,.
\end{align*}
It suffices to find $d$ such that $g(d) \deq d^2 \log d \ge 8\log(1/\varepsilon)$. In order to invert $g$, let $y$ be sufficiently large and consider finding $x$ such that $g(x) = y$. We make the choice $x = \alpha \sqrt{y/(\log y)}$ and plug this into the expression for $g$ in order to obtain
\begin{align*}
    \log g\Bigl( \alpha\sqrt{\frac{y}{\log y}}\Bigr)
    &= 2\log\alpha + \log y - \log \log y + \log \log\Bigl(\alpha\sqrt{\frac{y}{\log y}}\Bigr) \\
    &= 2\log \alpha + \log y + {\underbrace{\log \frac{\frac{1}{2} \log y - \frac{1}{2} \log \log y + \log \alpha}{\log y}}_{\to \log(1/2)~\text{as}~y\to\infty}}\,.
\end{align*}
From this expression, we see that provided $y$ is sufficiently large, this expression is less than $\log y$ for $\alpha =0$ and greater than $\log y$ for $\alpha = 3$. We conclude that $g^{-1}(y) \asymp \sqrt{y/(\log y)}$, and therefore that our choice of $d$ satisfies
\begin{align*}
    d
    &\asymp \sqrt{\frac{\log(1/\varepsilon)}{\log \log(1/\varepsilon)}}\,.
\end{align*}
In particular, since $d = o(\log(1/\varepsilon))$, then the condition $\varepsilon \le \exp(-c_\varepsilon d)$ holds for all sufficiently small $\varepsilon$, and Theorem~\ref{thm:bump_high_dim} holds.
Then,
\begin{align*}
    f(d)
    &\ge -\frac{d}{2} \log\log\frac{1}{\varepsilon} - \frac{d}{2} \log\frac{1}{c}
    \asymp -\sqrt{\bigl(\log \frac{1}{\varepsilon}\bigr) \, \bigl( \log \log \frac{1}{\varepsilon}\bigr)}\,.
\end{align*}
This verifies the expression in Section~\ref{scn:second_lower_bd_main_text}.

To justify the simplified expression of the bound that we gave in the informal statement of Theorem~\ref{thm:second_lower_informal}, note that in dimension
\begin{align}\label{eq:d_cond}
    d\lesssim \sqrt{\frac{\log(1/\varepsilon)}{\log \log(1/\varepsilon)}}
\end{align}
we have
\begin{align*}
    \log\Bigl( \bigl( \frac{cd}{\log(1/\varepsilon)} \bigr)^{d/2} \Bigr)
    &= \frac{d}{2} \, \underbrace{\Bigl( \log d - \log \log \frac{1}{\varepsilon} - \log \frac{1}{c} \Bigr)}_{\text{negative as}~\varepsilon \searrow 0}
    \gtrsim -\sqrt{\bigl(\log\frac{1}{\varepsilon}\bigr) \, \bigl( \log\log\frac{1}{\varepsilon}\bigr)}\,.
\end{align*}
In other words, we can simplify our bound as follows.
For all $d\ge 1$ and all $\varepsilon$ smaller than a universal constant, if the condition~\eqref{eq:d_cond} holds, then we have the lower bound
\begin{align*}
    \ms C(d,1,\varepsilon)
    &\gtrsim \frac{1}{\varepsilon^{2d/(d+2)} \exp(C\sqrt{\log(1/\varepsilon) \log\log(1/\varepsilon)})}\,.
\end{align*}
Otherwise, if the condition~\eqref{eq:d_cond} fails, then we instead have the bound
\begin{align*}
    \ms C(d,1,\varepsilon)
    &\gtrsim \frac{1}{\varepsilon^{2} \exp(C\sqrt{\log(1/\varepsilon) \log\log(1/\varepsilon)})}
    \ge \frac{1}{\varepsilon^{2d/(d+2)} \exp(C\sqrt{\log(1/\varepsilon) \log\log(1/\varepsilon)})}\,.
\end{align*}
In either case, we have $\ms C(d,1,\varepsilon) \ge {(1/\varepsilon)}^{2d/(d+2) - o(1)}$. Together with Theorem~\ref{thm:second_lower_1d} on the univarate case and Lemma~\ref{lem:rescaling} on rescaling, it yields Theorem~\ref{thm:second_lower_informal}.

\subsection{Proof of Theorem~\ref{thm:second_lower_1d}}\label{scn:pf_second_lower_1d}

In the univarate case, we can sharpen Theorem~\ref{thm:bump_high_dim} by obtaining a better bound on the Poincar\'e constant of $\pi_\omega$.
We use the following result.

\begin{theorem}[Muckenhoupt's criterion]\label{thm:muckenhoupt}
    Let $\pi$ be a probability density on $\R$ and let $m$ be a median of $\pi$. Then,
    \begin{align*}
        \CPI(\pi)
        &\asymp \max\Bigl\{ \sup_{x < m} \pi\bigl((-\infty, x]\bigr) \int_x^m \frac{1}{\pi}, \;\sup_{x > m} \pi\bigl([x,+\infty)\bigr) \int_m^x \frac{1}{\pi}\Bigr\}\,.
    \end{align*}
\end{theorem}
\begin{proof}
    See~\citet[Theorem 4.5.1]{bakrygentilledoux2014}.
\end{proof}

\begin{lemma}[improved Poincar\'e inequality for $\pi_\omega$]\label{lem:improved_poincare}
    Suppose that $d=1$ and $R\ge 1$. Then, for all $\omega \in \ms P_{2r,R}$,
    \begin{align*}
        \CPI(\pi_\omega)
        &\lesssim R^2\,.
    \end{align*}
\end{lemma}
\begin{proof}
    We use Muckenhoupt's criterion (Theorem~\ref{thm:muckenhoupt}). First, we note that by Property~\ref{p1}, it holds that $\pi_\omega(\omega + B_r) = \frac{1}{2}$ which implies that $\omega - r \le m \le \omega + r$. We proceed to check that
    \begin{align*}
        \sup_{x > m} \pi_\omega\bigl([x,+\infty)\bigr) \int_m^x \frac{1}{\pi_\omega} \lesssim R^2\,.
    \end{align*}
    The other condition is verified in the same way due to symmetry.
    
    We split into three cases.
    First, suppose that $m < x < \omega + r$. Then, as in the proof of Corollary~\ref{cor:poincare_pi_w}, we have $Z_\omega \le 2\,\Zinit = 2\, \tpinit(\R\setminus B_R) + 4R \le 2\sqrt{2\uppi} + 4R \lesssim R$. Then,
    \begin{align*}
        \pi_\omega\bigl([x,+\infty)\bigr) \int_m^x \frac{1}{\pi_\omega}
        &\le Z_\omega\,(x-m)
        \lesssim Rr
        \le R^2\,.
    \end{align*}
    
    Next, suppose that $\omega + r < x < R$. Then,
    \begin{align*}
        \pi_\omega\bigl([x,+\infty)\bigr) \int_m^x \frac{1}{\pi_\omega}
        &= \tilde\pi_\omega\bigl([x,+\infty)\bigr) \int_m^x \frac{1}{\tilde\pi_\omega}
        \le \Bigl( R-x + \sqrt{\frac{\uppi}{2}}\Bigr) \, (x-m)
        \lesssim R^2\,.
    \end{align*}
    
    Finally, suppose that $x > R$. Then, using standard Gaussian tail bounds,
    \begin{align*}
        &\pi_\omega\bigl([x,+\infty)\bigr) \int_m^x \frac{1}{\pi_\omega}
        = \tilde\pi_\omega\bigl([x,+\infty)\bigr) \int_m^x \frac{1}{\tilde\pi_\omega} \\
        &\qquad \le \Bigl[\sqrt{2\uppi} \, \bigl( \frac{1}{2} \wedge \frac{1}{x-R} \bigr) \exp\Bigl(-\frac{{(x-R)}^2}{2}\Bigr)\Bigr] \, \Bigl[ R - m + (x-R) \exp\Bigl( \frac{{(x-R)^2}}{2}\Bigr) \Bigr]\,.
    \end{align*}
    If $x-R \le 1$, then this yields
    \begin{align*}
        \pi_\omega\bigl([x,+\infty)\bigr) \int_m^x \frac{1}{\pi_\omega}
        &\lesssim R\,.
    \end{align*}
    Otherwise, if $x-R \ge 1$, then we obtain
    \begin{align*}
        \pi_\omega\bigl([x,+\infty)\bigr) \int_m^x \frac{1}{\pi_\omega}
        &\lesssim \frac{1}{x-R}\exp\Bigl(-\frac{{(x-R)}^2}{2}\Bigr) \, \Bigl[R + (x-R)\exp\Bigl( \frac{{(x-R)^2}}{2}\Bigr) \Bigr] \lesssim R\,.
    \end{align*}
    This completes the proof.
\end{proof}

We now use the improved Poincar\'e inequality in order to establish Theorem~\ref{thm:second_lower_1d}.

\medskip{}

\begin{proof}[Proof of Theorem~\ref{thm:second_lower_1d}]
    We follow the proof of Theorem~\ref{thm:bump_high_dim}.
    The proofs of Properties~\ref{p1} and~\ref{p3} remain unchanged.
    
    In the proof of Property~\ref{p2}, the equation~\eqref{eq:p2_proof} is replaced by
    \begin{align*}
        \TV^2(\mu,\pi_\omega)
        &\le \cPI R^2 \varepsilon^2
    \end{align*}
    for a different universal constant $\cPI > 0$, using Lemma~\ref{lem:improved_poincare}.
    Hence, we choose $R^2 = 1/(9\cPI\varepsilon^2)$ in order to verify Property~\ref{p2}. Since we require $R\ge c_R$ for a universal constant $c_R\ge 1$, this requires $\varepsilon \le \exp(-c_\varepsilon)$ for a universal constant $c_\varepsilon > 0$.
    
    Next, we turn towards the sharpened statement of Property~\ref{p4}. From~\eqref{eqn:R r def}, $r$ is chosen so that
    \begin{align*}
        (I_r + 2) \, r = \tilde\pi_\omega(\R\setminus B_R) + 2R\,.
    \end{align*}
    Using~\eqref{eqn:I_r bound} from Lemma~\ref{lem:high d properties}, we have
    \begin{align*}
        rI_r
        &\asymp \exp\bigl(r^2 \phi(0)\bigr)
        \gtrsim r\,.
    \end{align*}
    This implies that
    \begin{align*}
        \exp\bigl(r^2 \phi(0)\bigr)
        &\gtrsim (I_r + 2) \, r
        \gtrsim R\,,
    \end{align*}
    or $r \gtrsim \sqrt{\log R} \asymp \sqrt{\log(1/\varepsilon)}$. Hence,
    \begin{align*}
        \abs{\ms P_{2r,R}}
        &\ge \frac{R}{4r}
        \gtrsim \frac{1}{\varepsilon \sqrt{\log(1/\varepsilon)}}\,.
    \end{align*}
    
    By substituting this new bound on the packing number into the information theoretic argument of Theorem~\ref{thm:bump_high_dim} (see~\eqref{eq:est_lw}, where $M=\abs{\ms P_{2r,R}}$), we obtain Theorem~\ref{thm:second_lower_1d}.
\end{proof}

\section{Further discussion of the univariate case}\label{scn:univariate_discussion}

In this section, we provide further discussion of algorithms for the univariate case.

\paragraph{Rejection sampling.}
First of all, we note that the $\poly(1/\varepsilon)$ lower bounds of Theorems~\ref{thm:bump_high_dim} and~\ref{thm:second_lower_1d} may come as a surprise due to the existence of the rejection sampling algorithm. We briefly recall rejection sampling here. Let $\tilde \pi$ be an unnormalized density, let $Z_\pi \deq \int \tilde \pi$ denote the normalizing constant, and let $\pi \deq \tilde\pi/Z$ denote the target distribution.
Rejection sampling requires knowledge of an upper envelope $\tilde \mu$ for $\tilde\pi$, i.e., a function $\tilde\mu$ satisfying $\tilde\mu\ge\tilde\pi$ pointwise.
The algorithm proceeds by repeatedly drawing samples from the density $\mu \deq \tilde\mu/Z_\mu$, where $Z_\mu\deq \int \tilde \mu$; each sample $X$ is accepted with probability $\tilde \pi(X)/\tilde \mu(X)$.

It is standard to show~\citep[see, e.g.,][]{chewietal2022logconcave1d} that the accepted samples are drawn exactly from the target $\pi$, and that the number of queries made to $\tilde\pi$ until the first accepted sample is geometrically distributed with mean $Z_\mu/Z_\pi$. To translate this into a total variation guarantee, we run the algorithm for $N$ iterations and output ``FAIL'' if we have not accepted a sample by iteration $N$. The probability of failure is at most ${(1-Z_\pi/Z_\mu)}^N$, so the number of iterations required for the output of the algorithm to be $\varepsilon$-close to the target $\pi$ in total variation distance is $N \ge \log(1/\varepsilon)/\log(1-Z_\pi/Z_\mu)$.

Although this is a total variation guarantee, rather than a Fisher information guarantee, it suggests (similarly to Appendix~\ref{scn:separation}) that $\log(1/\varepsilon)$ rates are attainable using rejection sampling. The reason why this does not contradict our lower bounds in Theorems~\ref{thm:bump_high_dim} and~\ref{thm:second_lower_1d} is that the initialization oracle we consider, which provides a measure $\mu_0$ such that $\KL(\mu_0 \mmid \pi) \le K_0$, is not sufficient to construct an upper envelope of the unnormalized density $\tilde\pi$.

Indeed, consider instead a stronger initialization oracle which outputs a measure $\mu_0$ such that
\begin{align*}
    \max\Bigl\{\sup \ln \frac{\mu_0}{\pi},\; \sup\ln\frac{\pi}{\mu_0}\Bigr\}
    &\le M_0
    < \infty\,.
\end{align*}
Denote the complexity of obtaining $\sqrt{\FI(\mu \mmid \pi)} \le \varepsilon$ over the class of $1$-log-smooth distributions on $\R^d$ with this stronger initialization oracle by $\ms C_\infty(d, M_0,\varepsilon)$. Then, the rejection sampling algorithm can be implemented within this new oracle model. It yields the following.

\begin{proposition}[Fisher information guarantees via rejection sampling]\label{prop:rej_sampling_fi}
    It holds that
    \begin{align*}
        \ms C_\infty(d, M_0, \varepsilon)
        &\le \widetilde{\mc O}\Bigl( \exp(3M_0) \log\frac{\sqrt d}{\varepsilon}\Bigr)\,.
    \end{align*}
\end{proposition}
\begin{proof}
    For the algorithm, we use rejection sampling, which requires producing an upper envelope.
    Recall that in our oracle model, we can query the value of an unnormalized version $\tilde\pi$ of $\pi$. By replacing $\tilde \pi$ with $\tilde \pi/\tilde \pi(0)$, we can assume that $\tilde \pi(0) = 1$.
    Then,
    \begin{align*}
        \tilde \pi
        &= \frac{\tilde \pi}{\tilde \pi(0)}
        = \frac{\pi}{\pi(0)}
        \le \frac{\exp(M_0)\,\mu_0}{\exp(-M_0) \,\mu_0(0)}
        = \underbrace{\frac{\exp(2M_0)}{\mu_0(0)}}_{\deq Z_{\mu_0}} \, \mu_0\,.
    \end{align*}
    This shows that $\tilde \mu_0 \deq Z_{\mu_0} \, \mu_0$ is an upper envelope for $\tilde \pi$. Also, using $\pi(0) = 1/Z_\pi$,
    \begin{align*}
        \frac{Z_{\mu_0}}{Z_\pi}
        &= \exp(2M_0) \, \frac{\pi(0)}{\mu_0(0)}
        \le \exp(3M_0)\,.
    \end{align*}
    Hence, we can run rejection sampling, where we output a sample from $\mu_0$ if the algorithm exceeds $N$ iterations. Therefore, the law of the output of rejection sampling is $\mu = (1-p) \, \pi + p \, \mu_0$, where $p = {(1-Z_\pi/Z_{\mu_0})}^N \le \exp(-NZ_\pi/Z_{\mu_0})$ is the probability of failure. We calculate
    \begin{align*}
        1+\chi^2(\mu \mmid \pi)
        &= \E_{\mu} \Bigl(\frac{\mu}{\pi}\Bigr)
        = 1-p + p \E_\mu \Bigl(\frac{\mu_0}{\pi}\Bigr)
        \le 1 + p\exp(M_0)\,.
    \end{align*}
    Applying Lemma~\ref{l:FI-heat-bd} with $\varepsilon_\chi^2 = p\exp(M_0)$ (assuming that $p\le\exp(-M_0)$) and $t\lesssim 1$, we obtain
    \begin{align*}
        \FI(\mu Q_t \mmid \pi)
        &\lesssim \frac{p\exp(M_0) \, (d + \log(1/p) - M_0)}{t} + dt\,.
    \end{align*}
    We set $t \lesssim \varepsilon^2/d$ so that
    \begin{align*}
        \FI(\mu Q_t \mmid \pi)
        &\lesssim \frac{d^2 \exp(M_0) \, p\log(1/p)}{\varepsilon^2} + \varepsilon^2\,.
    \end{align*}
    In order to make the first term at most $\varepsilon^2/2$, we take $p = \widetilde{\Theta}(\varepsilon^4/(d^2 \exp(M_0)))$. In turn, this is satisfied provided
    \begin{align*}
        N
        &\ge \frac{Z_{\mu_0}}{Z_\pi} \log \frac{1}{p}
        \asymp \exp(3M_0) \log\frac{d^2 \exp(M_0)}{\varepsilon^4}\,,
    \end{align*}
    which proves the desired result.
\end{proof}

Hence, under the stronger oracle model, $\log(1/\varepsilon)$ rates are indeed possible (albeit with exponential dependence on $M_0$). To see why this does not contradict the lower bound construction of Theorem~\ref{thm:second_lower_1d}, observe that if we take the initialization oracle to be $\pinit$, then our construction satisfies $M_0 = r^2 \phi(0)$.
By inspecting the proof of Theorem~\ref{thm:second_lower_1d}, one sees that $r\asymp \sqrt{\log(1/\varepsilon)}$. Hence, our construction does not provide a lower bound for $\ms C_\infty(1,M_0, \varepsilon)$ for constant $M_0$. Instead, we obtain the following lower bound.

\begin{corollary}[lower bound for the stronger initialization oracle]\label{cor:stronger_oracle_lower_bd}
    There exists a universal constant $c > 0$ such that for all $\varepsilon \le 1/c$, it holds that
    \begin{align*}
        \ms C_\infty\bigl(1, c\log(1/\varepsilon), \varepsilon\bigr)
        &\gtrsim \frac{1}{\varepsilon\sqrt{\log(1/\varepsilon)}}\,.
    \end{align*}
\end{corollary}

Note also the following corollary.

\begin{corollary}[high-accuracy Fisher information requires exponential dependence on $M_0$]\label{cor:exponential_m0} \\
    Suppose that there exists an algorithm which works within the stronger oracle model and which, for any $1$-log-smooth distribution $\pi$ on $\R$, outputs a measure $\mu$ with $\sqrt{\FI(\mu\mmid \pi)} \le \varepsilon$ using $N$ queries, where the query complexity satisfies
    \begin{align*}
        N
        &\le f(M_0) \polylog\bigl(\frac{1}{\varepsilon}\bigr)
    \end{align*}
    for some increasing function $f : [1,\infty) \to \R_+$. Then, there is a universal constant $c' > 0$ such that
    \begin{align*}
        f(M_0)
        &\ge \widetilde \Omega\bigl( \exp(c' M_0)\bigr)\,.
    \end{align*}
\end{corollary}
\begin{proof}
    Using Corollary~\ref{cor:stronger_oracle_lower_bd} with $M_0 = c\log(1/\varepsilon)$, we have
    \begin{align*}
        f\bigl(c\log \frac{1}{\varepsilon}\bigr) \polylog\bigl(\frac{1}{\varepsilon}\bigr)
        &\ge N
        \gtrsim \frac{1}{\varepsilon \sqrt{\log(1/\varepsilon)}}\,,
    \end{align*}
    or
    \begin{align*}
        f\bigl(c\log\frac{1}{\varepsilon}\bigr)
        &\ge \frac{1}{\varepsilon \polylog(1/\varepsilon)}\,.
    \end{align*}
    Writing this in terms of $M_0 = c\log(1/\varepsilon)$, or $\varepsilon = \exp(-M_0/c)$,
    \begin{align*}
        f(M_0)
        &\ge \frac{\exp(M_0/c)}{{(M_0/c)}^{\mc O(1)}}
        = \widetilde \Omega\Bigl( \exp\bigl(\frac{M_0}{c}\bigr)\Bigr)
    \end{align*}
    which establishes the result.
\end{proof}

Hence, we see that there is a fundamental trade-off in the stronger oracle model: any algorithm must either incur polynomial dependence on $1/\varepsilon$ (e.g., averaged LMC), or exponential dependence on $M_0$ (e.g., rejection sampling, see Proposition~\ref{prop:rej_sampling_fi}).

\paragraph{The stronger oracle model is strictly stronger.} We also observe the following consequence of these observations. On one hand, our lower bound in Thoerem~\ref{thm:second_lower_1d} shows that
\begin{align*}
    \ms C(1, K_0 = 1, \varepsilon)
    &\ge \Omega\Bigl(\frac{1}{\varepsilon \sqrt{\log(1/\varepsilon)}}\Bigr)\,.
\end{align*}
On the other hand, for constant $M_0$, rejection sampling (Proposition~\ref{prop:rej_sampling_fi}) yields
\begin{align*}
    \ms C_\infty(1, M_0, \varepsilon)
    &\le \widetilde{\mc O}\Bigl(\exp(3M_0) \log\frac{1}{\varepsilon}\Bigr)\,.
\end{align*}
Hence, the stronger oracle model is indeed stronger: \emph{obtaining Fisher information guarantees is \underline{strictly} easier with access to an oracle with bounded $M_0$, rather than an oracle with bounded $K_0$}.

\paragraph{On the effect of the radius of the effective support.} In our lower bound construction, the distributions are ``effectively'' supported on a ball of radius $R$, where $R$ scales with $1/\varepsilon$. Here, we show that this is in fact necessary, by showing that for any \emph{fixed} $d$ and $R$, it is possible to sample from such a distribution in Fisher information using $\mc O(\log(1/\varepsilon))$ queries.
The algorithm involves uses a simple grid search.

\begin{proposition}[sampling from bounded effective support]\label{prop:bdd_supp}
    Suppose that the target distribution $\pi\propto\exp(-V)$ on $\R^d$ has the following properties:
    \begin{enumerate}
        \item $V(0) = 0$.
        \item $V(x) = \frac{1}{2} \, {(\norm x - R)}_+^2$, for $\norm x \ge R$.
        \item $V$ is $1$-smooth.
    \end{enumerate}
    Then, there is an algorithm which outputs $\mu$ with $\sqrt{\FI(\mu\mmid \pi)} \le \varepsilon$ using $N$ queries to $(V,\nabla V)$, where the number of queries satisfies
    \begin{align*}
        N
        &\lesssim {(cR)^d} + \log \frac{\sqrt d}{\varepsilon}\,,
    \end{align*}
    where $c > 0$ is a universal constant.
\end{proposition}
\begin{proof}
    We use function approximation to build an upper envelope for $\tilde \pi \deq \exp(-V)$, and then apply rejection sampling. Namely, let $\ms N$ be a $1$-net of $B_R$, and for each $x\in B_R$ let $x_{\ms N}$ denote a closest point of $\ms N$ to $x$.
    Define the approximation
    \begin{align*}
        \widehat V(x)
        &\deq \begin{cases} \frac{1}{2} \, {(\norm x - R)}_+^2\,, & \norm x \ge R\,, \\
        V(x_{\ms N}) + \langle \nabla V(x_{\ms N}), x-x_{\ms N}\rangle - \frac{1}{2} \,\norm{x-x_{\ms N}}^2\,, & \norm x < R\,.
        \end{cases}
    \end{align*}
    By $1$-smoothness of $V$, we have $V \ge \widehat V$, so that if we let $\tilde \mu_0 \deq \exp(-\widehat V)$, then $\tilde \mu_0 \ge \tilde\pi$. Also, for $\norm x < R$, we have the bound
    \begin{align*}
        \tilde \mu_0(x)
        &= \exp\Bigl(-V(x_{\ms N}) - \langle \nabla V(x_{\ms N}), x-x_{\ms N}\rangle + \frac{1}{2} \,\norm{x-x_{\ms N}}^2\Bigr) \\
        &\le \exp\bigl(-V(x) +\norm{x-x_{\ms N}}^2\bigr)
        = \tilde \pi(x) \exp(\norm{x-x_{\ms N}}^2)
        \le \exp(1)\, \tilde \pi(x)\,,
    \end{align*}
    so that $Z_{\mu_0}/Z_\pi \lesssim 1$. We now perform rejection sampling using $N'$ iterations with upper envelope $\tilde \mu_0$, outputting a sample from $\mu_0$ if $N'$ iterations are exceeded. Tracing through the proof of Proposition~\ref{prop:rej_sampling_fi}, one can show that for the output $\mu$ of rejection sampling, it holds that $\FI(\mu Q_t \mmid \pi) \le \varepsilon^2$ for an appropriate choice of $t$.
    Moreover, the number of iterations of rejection sampling required to achieve this satisfies $N' \lesssim \log(\sqrt d/\varepsilon)$. Finally, since $\abs{\ms N} \le {(cR)}^d$ for a universal constant $c > 0$, it requires $\mc O({(cR)}^d)$ queries in order to build the upper envelope $\tilde \mu_0$, which proves the result.
\end{proof}

To summarize the situation, if the effective radius $R$ is known and fixed, then it is possible to obtain $\mc O(\log(1/\varepsilon))$ complexity. However, if there is no a priori upper bound on the radius $R$, then the lower bounds of Theorem~\ref{thm:second_lower_1d} and Corollary~\ref{cor:stronger_oracle_lower_bd} apply.

\bibliography{ref.bib}

\end{document}